\documentclass[12pt,oneside,reqno]{amsart}

\usepackage{algorithm, algpseudocode, amsfonts, amssymb, amsthm, amsmath, array, authblk, bm, braket, caption, changepage, comment, dsfont, enumitem, etoolbox, fullpage, graphicx, mathrsfs, mathtools, microtype, multicol, pythonhighlight, subfig, subfloat, textcomp, tikz, tikz-cd, url, xpatch}

\usepackage{hyperref}
\hypersetup{
	colorlinks=true,
	linkcolor=blue,
	citecolor=black,
	urlcolor=cyan}
\usepackage{cleveref}

\setlist[enumerate]{itemsep=0pt, topsep=0pt}
\setlist[itemize]{itemsep=0pt, topsep=0pt}

\usepackage[foot]{amsaddr}

\usepackage{newpxtext} \usepackage[euler-digits]{eulervm}

\usepackage{titlesec}

\titleformat{\section}
	{\titlerule\vspace{-2ex}}
	{\thesection.}{1ex}{\centering\scshape}
	[\vspace{.5ex}\titlerule]
	
\titleformat{\subsection}[runin]
  {\normalfont\normalsize\bfseries}{\thesubsection.}{1ex}{}	
  
\titlespacing*{\subsection}{0pt}{0.0\baselineskip}{1ex}

\newtheorem{theorem}{Theorem}[section]

\newtheoremstyle{style}
  {\baselineskip} 
  {0em} 
  {\itshape} 
  {} 
  {\bfseries} 
  {.} 
  {.5em} 
  {} 
\theoremstyle{style}

\newtheorem*{theorem*}{Theorem}

\numberwithin{equation}{section}

\newtheorem*{definition*}{Definition}
\newtheorem{example}[theorem]{Example}

\newtheorem*{lemma*}{Lemma}
\newtheorem{prop}[theorem]{Proposition}
\newtheorem*{prop*}{Proposition}
\newtheorem{remark}[theorem]{Remark}

\algrenewcommand\algorithmicrequire{\textbf{Input:}}
\algrenewcommand\algorithmicensure{\textbf{Output:}}

\DeclareMathOperator*{\argmin}{arg\,min}
\DeclareMathOperator{\Cay}{Cay}

\DeclareMathOperator{\diag}{diag}

\title{{\large {\Large E}llipsoid fitting with the {\Large C}ayley transform}}

\author{{\large O}mar {\large M}elikechi\textsuperscript{1}}
\author{{\large D}avid {\large B}. {\large D}unson\textsuperscript{1,2}}
\address{\textsuperscript{1}Department of Statistical Science, Duke University, Durham, NC}
\address{\textsuperscript{2}Department of Mathematics, Duke University, Durham, NC}

\begin{document}

\maketitle

\begin{abstract}
We introduce Cayley transform ellipsoid fitting (CTEF), an algorithm that uses the Cayley transform to fit ellipsoids to noisy data in any dimension. Unlike many ellipsoid fitting methods, CTEF is ellipsoid specific, meaning it always returns elliptic solutions, and can fit arbitrary ellipsoids. It also significantly outperforms other fitting methods when data are not uniformly distributed over the surface of an ellipsoid. Inspired by growing calls for interpretable and reproducible methods in machine learning, we apply CTEF to dimension reduction, data visualization, and clustering in the context of cell cycle and circadian rhythm data and several classical toy examples. Since CTEF captures global curvature, it extracts nonlinear features in data that other machine learning methods fail to identify. For example, on the clustering examples CTEF outperforms $\bm{10}$ popular algorithms.
\end{abstract}


\section{Introduction}\label{sec:intro}


The problem of fitting ellipsoids to data has applications to statistics, data visualization, feature extraction, pattern recognition, computer vision, medical imaging, and robotics \cite{kesaniemi2018, han2019, lin2016, li2004}. Existing fitting methods either do not apply in more than 3 dimensions, cannot fit arbitrary ellipsoids, are not \textit{ellipsoid specific} (meaning they may return solutions that are not ellipsoids), are not invariant under translations and rotations of data, and/or do not perform well when data are noisy or not uniformly distributed over the surface of an ellipsoid. In this paper we introduce \textit{Cayley transform ellipsoid fitting} (CTEF), an ellipsoid fitting method that does all of the above and show via extensive experiments that it often performs better than existing algorithms. In particular, CTEF outperforms other methods when data are noisy or concentrate near a relatively small subset of an ellipsoid rather than uniformly over it, as is common in practice. One disadvantage of CTEF is it is slower than the fastest existing method, but this is likely prohibitive only in real-time applications such as adaptive calibration of 3D sensors \cite{pylvanainen2008}. On both real and simulated data CTEF routinely converges in under $1$ second running Python on a standard 2019 MacBook Air (\Cref{sec:runtime}).


\subsection*{Related work.}\label{sec:related}


Ellipsoid fitting algorithms are classified as \textit{geometric} or \textit{algebraic} \cite{markovsky2004}. Geometric methods attempt to solve a highly non-convex optimization problem fraught with computational and stability issues, especially when data are noisy or \textit{nonuniform}\footnote{Throughout this paper \textit{nonuniform} means data concentrate near a relatively small subset of an ellipsoid rather than uniformly over it.} \cite{lin2016}. CTEF and all other methods considered in this paper are algebraic. We introduce and discuss each of these in \Cref{sec:experiments}.  

\textit{Ellipsoid detection} is a related but different problem from ellipsoid fitting common in computer vision. The former aims to find elliptic patterns in data, while the latter aims to determine the ellipsoid parameters that best fit the given data, usually by minimizing some loss. The generalized Hough transform \cite{ballard1981} is a popular tool for ellipsoid detection, but its memory and computational demands make it largely intractable for fitting ellipsoids in more than $2$ dimensions \cite{antolovic2008,basca2005,bennett1999}. Some progress has been made toward a practical Hough transform-based method for $3$-dimensional ellipsoids, but these are restricted in the types of ellipsoids they can fit \cite{hsu1990,camurri2014}.


\subsection*{Organization of paper.}\label{sec:organization}


In \Cref{sec:method} we introduce CTEF and establish key properties. In \Cref{sec:experiments} we compare CTEF to other ellipsoid fitting methods in multiple settings and empirically assess parameter sensitivity and runtime. In \Cref{sec:applications} we apply CTEF to dimension reduction, data visualization, and clustering. This builds in part on work in \cite{paul2020} which is limited by a method that only fits ellipsoids whose axes are parallel to the standard coordinate axes in Euclidean space. As noted there, many dimension reduction and clustering algorithms are highly complex and heavily parametrized, yielding results that are difficult to interpret or reproduce. In contrast CTEF both captures nonlinear structure and admits a clear interpretation. Indeed, letting $\lVert\cdot\rVert$ be Euclidean norm, every ellipsoid $\mathcal{E}$ in $\mathbb{R}^p$ satisfies
\begin{equation}\label{eq:ellipsoid}
\begin{aligned}
	\mathcal{E} &= \left\{x\in\mathbb{R}^p : \lVert A(a)R(x-c)\rVert^2 = 1\right\} \\
		&= \left\{ R^TA(a)^{-1}\eta + c : \eta\in S^{p-1}\right\}
\end{aligned}
\end{equation}
for some $p$-by-$p$ diagonal matrix $A(a)$ with diagonal entries $a_i>0$, a special orthogonal matrix\footnote{$SO(p)=\{R\in\mathbb{R}^{p\times p}:R^TR=RR^T=I, \det(R)=1\}$.} $R\in SO(p)$, and $c\in\mathbb{R}^p$. The reciprocals $1/a_i$ are lengths of the principal axes of $\mathcal{E}$, $R$ is a rotation that determines its orientation, and $c$ is its center. Equivalently, $\mathcal{E}$ is obtained from the unit sphere $S^{p-1}=\{\eta\in\mathbb{R}^p:\lVert \eta\rVert=1\}$ by \textit{scaling} $S^{p-1}$ by $A^{-1}$, \textit{rotating} by $R^T$, and \textit{translating} by $c$. CTEF reliably fits $A$, $R$, and $c$, giving a concrete relationship between noisy, nonlinear data and the simple object that is the sphere. The ideas in \Cref{sec:applications} thus stand to contribute to the growing call for interpretable algorithms in machine learning \cite{molnar2020,rudin2019}. In \Cref{sec:discussion} we elaborate this point and discuss future work.


\section{Cayley transform ellipsoid fitting}\label{sec:method}


In this section we introduce CTEF and prove it is (1) ellipsoid specific (\Cref{sec:loss}), (2) able to fit arbitrary ellipsoids (\Cref{rmk:cayley}), (3) invariant under rotations and translations of data (\Cref{sec:invariance}), and (4) convergent (\Cref{sec:convergence}). CTEF minimizes a loss (\Cref{sec:loss}) over a \textit{feasible set} (\Cref{sec:feasible}), yielding parameters that define our ellipsoid of best fit (\Cref{eq:objective}). The main idea in what follows is that the Cayley transform, defined below, turns an optimization problem with difficult nonlinear constraints to a problem with simple linear ones, opening the door to a range of optimization algorithms. We use the \textit{STIR} (\textit{Subspace Trust region Interior Reflective}) trust region method introduced in \cite{branch1999} and implemented by the Python package \cite{scipy2020}. STIR is specifically designed for bound-constrained nonlinear minimization problems in Euclidean space, making it well-suited to objective \eqref{eq:objective}. \Cref{alg:fit} outlines CTEF; definitions and details are in the subsections that follow.

\begin{algorithm}
\caption{Cayley transform ellipsoid fitting (CTEF)}\label{alg:fit}
\begin{algorithmic}[1]
\Require{Data $X\in\mathbb{R}^{n\times p}$}
\State $Y\gets V^T(X-\bar{x})$ \hfill\Comment{Transform data}
\State $\mathcal{F}\gets \mathcal{F}(Y)$ \hfill\Comment{Feasible set $\mathcal{F}$ based on $Y$}
\State $(a,c,s)\gets STIR(\mathcal{L},Y,\mathcal{F})$ \hfill\Comment{Minimize $\mathcal{L}$ over $\mathcal{F}$}
\Ensure{Diagonal matrix $A(a)$, rotation $R(s)$, center $c$}
\end{algorithmic}
\end{algorithm}


\subsection{The loss.}\label{sec:loss}


Ellipsoids in $\mathbb{R}^p$ are often expressed as
\begin{align}\label{eq:ellipsoid_M}
	\mathcal{E} &= \left\{x\in\mathbb{R}^p : (x-c)^TM(x-c)=1\right\}
\end{align}
where $M$ is a symmetric positive definite matrix and $c$ is the ellipsoid center. Eigendecomposing $M=R^TA^2R$ with $A$ diagonal and $R\in SO(p)$ shows \eqref{eq:ellipsoid_M} is equivalent to \eqref{eq:ellipsoid}. The positive definite constraint presents significant difficulty for fitting methods and is why most cannot fit arbitrary ellipsoids or are not ellipsoid specific \cite{kesaniemi2018, szpak2012}. To bypass this issue we use the \textit{Cayley transform} \cite{cayley1846} which maps skew-symmetric matrices $S$ to special orthogonal matrices via $\Cay(S)=(I+S)^{-1}(I-S)$. This is well-defined because $I+S$ is always invertible: If $x$ is in the nullspace of $I+S$ then $x=-Sx=S^Tx$, so $x^Tx=x^TS^Tx=-x^Tx$ and hence $x=0$. 

Consider fitting an ellipsoid to data $X=\{x^{(i)}\}_{i=1}^n\subseteq\mathbb{R}^p$. To ensure invariance under rotations and translations (\Cref{sec:invariance}) we transform $X$ to $\Phi_X(X)=\{\Phi_X(x) : x\in X\}$ where $\Phi_X(x)=V^T(x-\bar{x})$ with $\bar{x}=n^{-1}\sum x_i$ the mean of $X$ and $V\in O(p)$ an orthogonal matrix\footnote{$O(p)=\{Q\in\mathbb{R}^{p\times p}: Q^TQ=QQ^T=I\}$.} whose columns are eigenvectors\footnote{Also known as the \textit{principal components} of $X$.} of the covariance matrix of the centered data $\{x^{(i)}-\bar{x}\}$. Eigenvalues are always assumed to be ordered from largest to smallest with columns of $V$ ordered accordingly. Geometrically, $\Phi_X$ is an isometry mapping $X$ to a reference frame that is independent of its original orientation. We focus on fitting an ellipsoid to $\Phi_X(X)$; the best fit ellipsoid for $X$ is then readily obtained via $\Phi_X^{-1}$ (\Cref{sec:objective}).

Having introduced the Cayley transform, we now define our loss. Set $\mathbb{R}_+=(0,\infty)$ and $y^{(i)}=\Phi_X(x^{(i)})$. In light of \eqref{eq:ellipsoid} it is natural to consider minimizing $L:\mathbb{R}_+^p\times\mathbb{R}^p\times SO(p)\to \mathbb{R}_+$ defined by
\begin{align}\label{eq:constrained_loss}
	L(a,c,R) &= \sum_{i=1}^n \left(\lVert A(a)R(y^{(i)} - c)\rVert^2 - 1\right)^2,
\end{align}
but this has constraint $R\in SO(p)$. Instead, define $\mathcal{L}:\mathbb{R}_+^p\times\mathbb{R}^p\times\mathbb{R}^{p(p-1)/2}\to\mathbb{R}_+$ by
\begin{align}\label{eq:loss}
	\mathcal{L}(a,c,s) &= \sum_{i=1}^n \left(\lVert A(a)R(s)(y^{(i)} - c)\rVert^2 - 1\right)^2.
\end{align}
This is identical to \eqref{eq:constrained_loss} but now $R:\mathbb{R}^{p(p-1)/2}\to SO(p)$ is\footnote{Here $\circ$ denotes function composition, i.e. $f\circ g(x)=f(g(x))$.} $\Cay\circ S$ where $S$ is the identification of $\mathbb{R}^{p(p-1)/2}$ with the space $\mathfrak{so}(p)$ of skew-symmetric matrices, e.g.
\begin{align*}
	S(s_1,s_2,s_3) &=
		\begin{pmatrix}
			\phantom{-}0 & \phantom{-}s_1 & s_2 \\
			-s_1 & \phantom{-}0 & s_3 \\
			-s_2 & -s_3 & 0 
		\end{pmatrix}
\end{align*} 
when $p=3$. The Cayley transform has several well-established and desirable properties. First, it maps skew-symmetric matrices to special orthogonal matrices so any triple $(a,c,s)$ produces an ellipsoid of the form \eqref{eq:ellipsoid}, proving Claim (1). It is also a surjection onto
\begin{align*}
	SO(p)_* &= \left\{R\in SO(p) : -1\ \text{not an eigenvalue of } R\right\}
\end{align*}
so minimizing \eqref{eq:loss} can fit every ellipsoid except those in a set of measure zero. This is in stark contrast to existing ellipsoid specific methods that only fit ellipsoids with axis ratio\footnote{Throughout this paper, \textit{axis ratio} is the ratio between the largest and smallest axis lengths of an ellipsoid.} at most 2 or axes that are parallel to the standard coordinate axes \cite{kesaniemi2018,paul2020}. While the above proves an almost sure version of Claim (2), the full claim is justified by the following.

\begin{remark}\label{rmk:cayley}
It was shown in \cite{kahan2006} that every $R\in SO(p)$ satisfies $R=\Cay(S)D$ for some $S\in\mathfrak{so}(p)$ with entries $\lvert s_{ij}\rvert<1$ and diagonal matrix $D$ with diagonal entries $\pm 1$. See also \cite{helfrich2018}. Thus replacing $\Cay(S)$ with $\Cay(S)D$ in our minimization problem guarantees every ellipsoid can be obtained by such an objective, not just those in $SO(p)_*$. Moreover, the set $\mathbb{R}^{p(p-1)/2}$ in which the $s$ variables reside could be replaced by $[-1,1]^{p(p-1)/2}$. We do not pursue this direction however since numerical results suggest no need to include $D$ in practice.
\end{remark}
\vspace*{1em}
Proposition \ref{prop:gradient} gives the gradient of $\mathcal{L}$. This enables approximate differentiation in gradient-based algorithms like STIR to be replaced with an exact closed form expression, greatly improving computational performance.

\begin{prop}\label{prop:gradient}
Fix $x\in\mathbb{R}^p$ and define $\ell:\mathbb{R}^p_+\times\mathbb{R}^p\times\mathbb{R}^{p(p-1)/2}\to\mathbb{R}$ by $\ell(a,c,s)=\tfrac{1}{2}\lVert AR(x-c)\rVert^2$ with $A=\diag(a)$ and $R=\Cay(S(s))$ as above. Then
\begin{align*}
	\nabla_a\ell(a,c,s) &= a^T\diag\left(Ry\odot Ry\right) \\
	\nabla_c\ell(a,c,s) &= -y^TR^TA^2R \\
	\nabla_s\ell(a,c,s) &= B^T - B
\end{align*}
where $y=x-c$, $\odot$ is the Hadamard product, and $B=(I-S(s))^{-1}A^2Ryy^T(I+R^T)$.
\end{prop}

\begin{proof}
Note $\ell(a,c,s)=\tfrac{1}{2}y^TR^TA^2Ry$. The expressions for $\nabla_a\ell$ and $\nabla_c\ell$ follow immediately from standard differentiation rules. $\nabla_s\ell$ follows from \cite[Theorem 3.2]{helfrich2018} which states $\nabla_s\ell(s)=B^T-B$ with
\begin{align*}
	B &= (I-S(s))^{-1}\nabla_R\ell(I+R^T).
\end{align*}
Since $\nabla_R\ell = A^2Ryy^T$, the result holds.
\end{proof}

\Cref{ex:infinity} shows $\mathcal{L}$ is globally minimized as certain parameters go to infinity. This motivates our restriction to feasible sets described in \Cref{sec:feasible}.

\begin{example}\label{ex:infinity}
If all data lie exactly on $\mathcal{E}=\{x:\lVert A(a)R(s)(x-c)\rVert^2=1\}$, then clearly $a$, $c$, and $s$ are global minimizers of $\mathcal{L}$. But $\mathcal{L}$ is also globally minimized as $1/a$ and\footnote{For $a\in\mathbb{R}^p_+$ we define $1/a=(1/a_1,\dots,1/a_p)$.} $c$ go to infinity. To see why, consider the 2-dimensional example where $1/a=(m,m)$ and $c=(m,0)$ for $m\in\mathbb{N}$ and $s=0$ (so $R=I$). Direct computation shows that for any $x\in\mathbb{R}^2$,
\begin{align*}
	\lim_{m\to\infty}\lVert AR(x-c)\rVert^2 &= \lim_{m\to\infty}\frac{(m-x_1)^2 + x_2^2}{m^2}
		= 1.
\end{align*}
Hence $\lim_{m\to\infty}\mathcal{L}(a,c,s)=0$ whenever $n<\infty$. Intuitively, $\mathcal{L}$ is minimized by ellipsoids with infinite axis lengths and centers infinitely far from the data.
\end{example}


\subsection{The feasible set.}\label{sec:feasible}


To avoid infinite solutions (\Cref{ex:infinity}) and guarantee convergence (\Cref{sec:convergence}) we restrict the domain $\mathbb{R}^p_+\times\mathbb{R}^p\times\mathbb{R}^{p(p-1)/2}$ of $\mathcal{L}$ to a \textit{feasible set} $\mathcal{F}= [a^-,a^+]\times[c^-,c^+]\times [s^-,s^+]$ where $[a^-,a^+]$ is the rectangle $[a^-_1,a^+_1]\times\cdots\times [a^-_p, a^+_p]$ in $\mathbb{R}^p$ and $[c^-,c^+]$ and $[s^-,s^+]$ are defined similarly. Thus axis lengths $1/a_i$ satisfy $1/a_i^+\leq 1/a_i\leq 1/a_i^-$ and $c$ lies in $[c^-,c^+]$. Empirical studies indicate $s^-$ and $s^+$ are inconsequential provided $s^-_i\leq -1 < 1\leq s^+_i$, which is unsurprising given \Cref{rmk:cayley}. Our code uses default values $s^-_i=-5$ and $s^+_i=5$ for all $i$. The bound $a^+$ is also inconsequential provided it is sufficiently large; we set $a^+_i= 10^{300}$ for all $i$ so no axis has length less than $10^{-300}$. To choose $a^-$, $c^-$, and $c^+$, for $1\leq j\leq p$ define
\begin{align*}
	c^-_j &= \min_{1\leq i\leq n}y^{(i)}_j,\
	c^+_j = \max_{1\leq i\leq n}y^{(i)}_j,\
	m_j = c_j^+ - c_j^-.
\end{align*}
Our default for $a^-$ is $(w_{ax}\max_j m_j)^{-1}$ with $w_{ax}=10$, which is used for all experiments and examples in this paper except \Cref{fig:tau10,fig:noise10,fig:circadian_half}. In many settings CTEF is robust to $w_{ax}$ (\Cref{fig:sensitivity_waxis}). The default rectangle for $c$ is
\begin{align*}
	w[c^-,c^+] &= w[c_1^-,c_1^+] \times \cdots \times w[c_p^-,c_p^+]
\end{align*}
where $w>0$ determines the size of $\mathcal{F}$ (\Cref{fig:bounds}). While more significant than $w_{ax}$, $w$ is inconsequential when data are relatively uniform over the ellipsoid. However, $w$ can be significant when data are highly nonuniform or very noisy (\Cref{fig:cost} and \Cref{sec:sensitivity}). Furthermore, too small a $w$ can result in $w[c^-,c^+]$ not containing the true ellipsoid center, though this can always be avoided by increasing $w$. When unspecified, CTEF automatically computes $w$ with a linear regression model trained on $[c^-,c^+]$ from simulated data. While this works well in general (all experiments in \Cref{sec:results} choose $w$ this way), we recommend trying additional values for $w$ when possible, especially for highly nonuniform data.

\begin{figure*}
\makebox[\textwidth][c]{\includegraphics[width=.95\textwidth, height=.25\textheight]{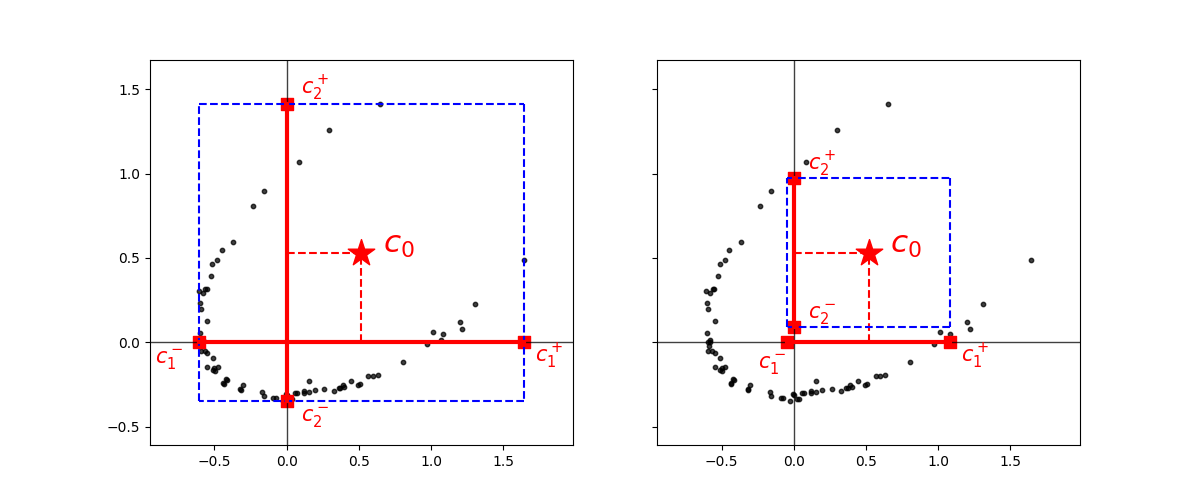}}%
\caption{\textit{{Feasible sets (blue) for the center when $w=1$ (left) and $w=0.5$ (right). Samples $X=\{x^{(i)}\}$ are drawn from the Ellipsoid-Gaussian distribution described in \Cref{sec:experiments} with $\tau=2$ and $1\%$ noise, then transformed by $\Phi_X$. The star $c_0$ marks the midpoint of $w[c^-,c^+]$ and is the default initial value in our algorithm. The initial value for $a$ is the vector of all ones. Note $c_0$ does not depend on $w$.}}}
\label{fig:bounds}
\end{figure*}


\subsection{The objective.}\label{sec:objective}


With $\mathcal{L}$ and $\mathcal{F}$ defined, our objective is to find
\begin{align}\label{eq:objective}
	(a_*,c_*,s_*) &= \argmin_{(a,c,s)\in\mathcal{F}}\mathcal{L}(a,c,s).
\end{align}
The resulting ellipsoid $\mathcal{E}=\{y\in\mathbb{R}^p:\lVert A(a_*)R(s_*)(y-c_*)\rVert=1\}$ is the best fit for the transformed data $Y=\Phi_X(X)$. Transforming back to $X=\{Vy+\bar{x}:y\in Y\}$ and making the substitution $x=Vy+\bar{x}$ gives
\begin{align*}
	\mathcal{E}_X &= V\mathcal{E} + \bar{x}
		= \{Vy+\bar{x} : y\in\mathcal{E}\} \\
		&= \left\{Vy+\bar{x} : \lVert A(a_*)R(s_*)(y-c_*)\rVert = 1\right\} \\
		&= \left\{x : \lVert A(a_*)R(s_*)V^T\left(x-(Vc_*+\bar{x})\right)\rVert^2=1\right\} \\
		&= \left\{x : \lVert A(a_*)\widetilde{R}(s_*)(x-\widetilde{c}_*)\rVert=1\right\}
\end{align*}
where $\widetilde{c}_*=Vc_*+\bar{x}$ and $\widetilde{R}(s_*)=R(s_*)V^T$. The second equality in \eqref{eq:ellipsoid} yields the equivalent form
\begin{align}\label{eq:fit}
	\mathcal{E}_X &= \{V\widetilde{R}(s_*)^TA(a_*)^{-1}\eta + \widetilde{c}:\eta\in S^{p-1}\}.
\end{align}
$\mathcal{E}_X$ is the \textit{ellipsoid of best fit for $X$}. This is justified since
\begin{align*}
	A(a)R(s)(y^{(i)}-c) &= A(a)R(s)V^T\left(x^{(i)}-(Vc+\bar{x})\right) \\
		&= A(a)\widetilde{R}(s)(x^{(i)}-\widetilde{c}),
\end{align*}
 and hence $L(a,c,R(s))=L(a,\widetilde{c},\widetilde{R}(s))$ for all $a$, $c$, and $s$, where $L$ is the loss \eqref{eq:constrained_loss}. In particular, $(a_*,c_*,s_*)$ satisfies \eqref{eq:objective} if and only if $(a_*, c_*,R(s_*))$ minimizes $L$ over $\mathcal{F}$ if and only if $(a_*, \widetilde{c}_*,\widetilde{R}(s_*))$ minimizes $L$ over the feasible set $\widetilde{\mathcal{F}}=\{(a, \widetilde{c}, s) : (a,c,s)\in\mathcal{F}\}$, which is just a rotation and translation of the $c$ coordinates in $\mathcal{F}$.


\subsection{Invariance.}\label{sec:invariance}


It is desirable for ellipsoid fitting algorithms to be invariant under rotations and translations \cite{bookstein1979, fitzgibbon1999, kesaniemi2018}. Formally, an algorithm is \textit{invariant} if $\mathcal{E}_Z=Q\mathcal{E}_X+\tau$ for any $X=\{x^{(i)}\}\subseteq\mathbb{R}^p$, $Q\in O(p)$, and $\tau\in\mathbb{R}^p$, where $\mathcal{E}_X$ and $\mathcal{E}_Z$ are the ellipsoids of best fit for $X$ and the transformed data $Z=\{z^{(i)}=Qx^{(i)}+\tau\}$, respectively.

\begin{prop}\label{prop:invariant}
CTEF is invariant. 
\end{prop}

\begin{proof}
Let $X$ and $Z$ be as above. Direct computation shows the mean of $Z$ is $\bar{z}=Q\bar{x}+\tau$ and the columns of $W=QV$ are eigenvectors of the covariance matrix of $\{z^{(i)}-\bar{z}\}$. Therefore
\begin{equation}\label{eq:invariant}
\begin{aligned}
	\Phi_Z(z^{(i)}) &= W^T(z^{(i)}-\bar{z}) \\
		&= V^TQ^T(Qx^{(i)}+\tau-Q\bar{x}-\tau) \\
		&= V^T(x^{(i)}-\bar{x})		
		= \Phi_X(x^{(i)}).
\end{aligned}
\end{equation}
Hence $\Phi_X(X)=\Phi_Z(Z)$ and the ellipsoid $\mathcal{E}$ obtained by solving objective \eqref{eq:objective} is identical for $X$ and $Z$. Thus $\mathcal{E}_Z = W\mathcal{E} +\bar{z} = QV\mathcal{E} + Q\bar{x} + \tau = Q\mathcal{E}_X+\tau$.
\end{proof}

\begin{remark}
The feasible set $\mathcal{F}$ depends only on $\Phi_X(X)$. So since $\Phi_X(X)=\Phi_Z(Z)$ for any rotation and translation $Z$ of $X$, invariance of CTEF is unaffected by the choice of $\mathcal{F}$. Observe also that $X$ in the preceding proof is an arbitrary finite subset of $\mathbb{R}^p$, so CTEF is invariant regardless of model assumptions, outliers, or noise. Finally, in \Cref{sec:applications} we use CTEF for dimension reduction by replacing $V$ with the $p$-by-$k$ matrix $V_k$ consisting of $k< p$ columns of $V$ and solving \eqref{eq:objective} to obtain a best fit ellipsoid $\mathcal{E}_k$ for $\Phi_X^k(X)\subseteq\mathbb{R}^k$ where $\Phi_X^k(x)=V_k^T(x-\bar{x})$. The $k$-dimensional ellipsoid of best fit for $X$ is then $\mathcal{E}^k_X=V_k\mathcal{E}_k+\bar{x}$. CTEF remains invariant in this case since, letting $Z=QX+\tau$ as before, we have $W_k=QV_k$ and replacing $W$ and $V$ with $W_k$ and $V_k$ in \eqref{eq:invariant} shows $\Phi^k_X(X)=\Phi^k_Z(Z)$. Thus, as with $\mathcal{E}$ in the preceding proof, $\mathcal{E}_k$ is identical for $X$ and $Z$ and 
\begin{align*}
	\mathcal{E}^k_Z &= W_k\mathcal{E}_k +\bar{z}
		= QV_k\mathcal{E}_k + Q\bar{x} + \tau
		= Q\mathcal{E}^k_X+\tau.
\end{align*}
\end{remark}


\subsection{Convergence.}\label{sec:convergence}


As discussed at the beginning of this section, we use STIR to solve \eqref{eq:objective}. Boundedness of $\mathcal{F}$ guarantees convergence of STIR to a local minimum of the loss function $\mathcal{L}$. This is because the Cayley transform, and hence $\mathcal{L}$, is twice continuously differentiable \cite{hairer2022}. Since $\mathcal{F}$ is compact it must therefore contain a minimum of $\mathcal{L}$. Furthermore, for any $\theta_0=(a_0,b_0,s_0)\in\mathcal{F}$ the level set
\begin{align*}
	\mathscr{L}(\theta_0) &= \left\{\theta = (a,b,s)\in\mathcal{F} : \mathcal{L}(\theta)\leq\mathcal{L}(\theta_0)\right\}
\end{align*}
is compact as it is closed\footnote{If $\theta_m$ is a sequence in $\mathscr{L}(\theta_0)$ converging to $\theta_*$, then $\mathcal{L}(\theta_*)=\lim_{m\to\infty}\mathcal{L}(\theta_m)\leq\mathcal{L}(\theta_0)$ and so $\theta_*\in\mathscr{L}(\theta_0)$.} by continuity of $\mathcal{L}$ and bounded by construction. Convergence of STIR to a local minimum then follows from \cite[Theorem 3]{branch1999}. See also \cite{coleman1996}.

\begin{remark}
A potential concern in light of \Cref{ex:infinity} is that the local minimum to which CTEF converges will often lie on the boundary of the feasible set $\mathcal{F}$, potentially resulting in a poor fit. Thankfully, STIR greatly diminishes this possibility by implementing a reflection technique that encourages exploration of the interior of $\mathcal{F}$ \cite{branch1999,coleman1996}. While boundary solutions are still possible, our experience shows they seldom occur in practice, especially for data distributed over substantial portions of the ellipsoid. Moreover, when boundary solutions do occur the fits still tend to be good.
\end{remark}


\section{Experiments}\label{sec:experiments}


We compare CTEF to $6$ methods: sum-of-discriminants (SOD) \cite{kesaniemi2018}, fixed constant (FC) \cite{rosin1993}, Bookstein (BOOK) \cite{bookstein1979}, hyperellipsoid-fit (HES) \cite{kesaniemi2018}, Taubin (TAUB) \cite{taubin1991}, and Calafiore plus alternating direction method of multipliers (CADMM) \cite{lin2016}. Only CTEF, HES, and CADMM are ellipsoid specific when $p>2$. HES cannot fit ellipsoids with axis ratio greater than 2 when $p>2$ and FC is not invariant under translations \cite{kesaniemi2018}. CADMM implements the algorithm introduced in \cite{calafiore2002} but uses ADMM to reduce its complexity. SOD, FC, BOOK, HES, and TAUB are implemented with the MATLAB package \cite{kesaniemi2023}. In all that follows SOD, FC, BOOK, and HES are regularized via the method introduced in \cite{kesaniemi2018} with regularization parameter selected according to \cite{gander1994}. Following \cite{kesaniemi2023} regularization is not used for TAUB.


\subsection{Data simulation.}\label{sec:simulation}


Data are simulated from the \textit{Ellipsoid-Gaussian model}
\begin{align}\label{eq:ellipsoid_gaussian}
	x^{(i)} &= \Lambda\eta_i + c + \epsilon_i
\end{align}
where $\Lambda=R^TA^{-1}$ with $c$, $A$, and $R$ as in \eqref{eq:ellipsoid}, $\eta_i\sim vMF(\mu,\tau)$ are independent random variables drawn from the von Mises-Fisher distribution on $S^{p-1}$, and $\epsilon_i\sim\mathcal{N}_p(0,\sigma^2 I)$ are independent normal random variables representing noise. Model \eqref{eq:ellipsoid_gaussian} is introduced and studied in detail in \cite{song2022}. The parameters of $vMF(\mu,\tau)$ are the mean $\mu\in S^{p-1}$ and measure of spread $\tau\geq 0$. Increasing $\tau$ increases the concentration of the distribution about $\mu$, with $\tau=\infty$ corresponding to a delta measure at $\mu$ and $\tau=0$ to the uniform distribution on $S^{p-1}$. As before, $c$ is the center of the ellipsoid and $\Lambda$ determines its shape and orientation. Entries of $c$ are drawn from $\mathcal{N}(0,10)$ in all experiments. $\Lambda$ is chosen to have determinant 1 and so that the resulting ellipsoid has a specified axis ratio. Other axis lengths are chosen uniformly at random between the smallest and largest. The standard deviation $\sigma$ is specified as a percentage of the diameter of the longest axis of the ellipsoid. Essentially all choices are made to agree with \cite{kesaniemi2018}.


\subsection{Measures of performance.}\label{sec:performance}


Several measures of ellipsoid fitting performance exist in the literature. The two we focus on are \textit{offset error}, $e_o$, and \textit{shape error}, $e_s$ \cite{kesaniemi2018}. Offset error is the norm $\lVert c - c_*\rVert$ of the difference between the estimated center, $c$, and the true center $c_*$ of the ellipsoid from which data are simulated. Shape error is
\begin{align*}
	e_s &= \frac{\sigma_1(\Lambda^{-1}\Lambda_*)}{\sigma_p(\Lambda^{-1}\Lambda_*)} - 1
\end{align*}
where $\Lambda$ and $\Lambda_*$ are the estimated and true shape matrices and $\sigma_1$ and $\sigma_p$ are the largest and smallest singular values of $\Lambda^{-1}\Lambda_*$. Another class of measures is
\begin{align}\label{eq:general_loss}
	\ell_{p,q}(\Lambda,c) &= \sum_{i=1}^n \left\lvert\lVert \Lambda^{-1}(x^{(i)}-c)\rVert^p-1\right\rvert^q
\end{align}
for integers $p$ and $q$ \cite{calafiore2002,lin2016}. Observe $\ell_{2,2}$ is precisely \eqref{eq:loss}. Unlike offset and shape errors, $\ell_{p,q}$ does not require knowledge of $\Lambda_*$ and $c_*$. However, $\ell_{p,q}$ error can be misleading since it goes to $0$ as axis lengths and centers go to infinity. Thus $\ell_{p,q}$ can be arbitrarily small for ellipsoid fits that are arbitrarily bad (\Cref{fig:cost}). For this reason we only consider offset and shape error in our experiments.

\begin{figure*}
\makebox[\textwidth][c]{\includegraphics[width=.65\textwidth, height=.2\textheight]{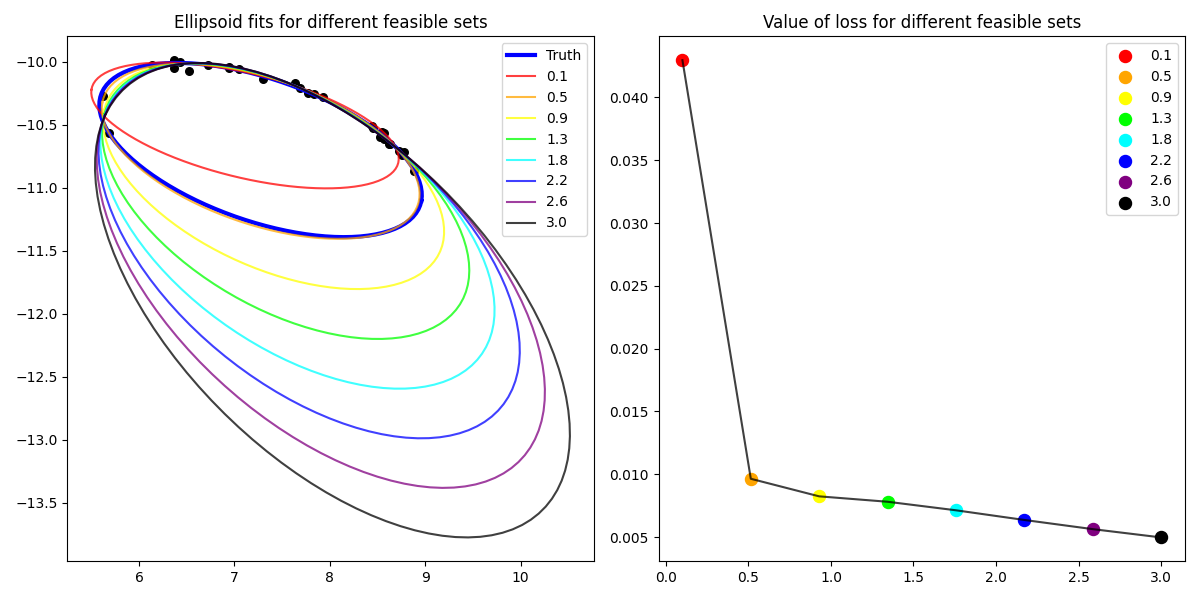}}%
\caption{\textit{(Left) Best fit ellipsoids for different feasible set weights $w$ defined in \Cref{sec:feasible}. Each color represents a different $w$. The best fit ellipsoid when $w=0.5$ (orange) closely resembles the true ellipsoid (thick blue curve). (Right) Loss corresponding to each $w$. For example, the loss when $w=0.5$ is approximately $0.01$. While loss decreases monotonically to $0$, centers and axis lengths of the fitted ellipsoids diverge. Here $n=30$ data points (black dots, left panel) are simulated from the Ellipsoid-Gaussian model with $\tau=2$ and axis ratio $3$.}}
\label{fig:cost}
\end{figure*}

\begin{figure*}
\centering
\includegraphics[width=\textwidth, height=.3\textheight]{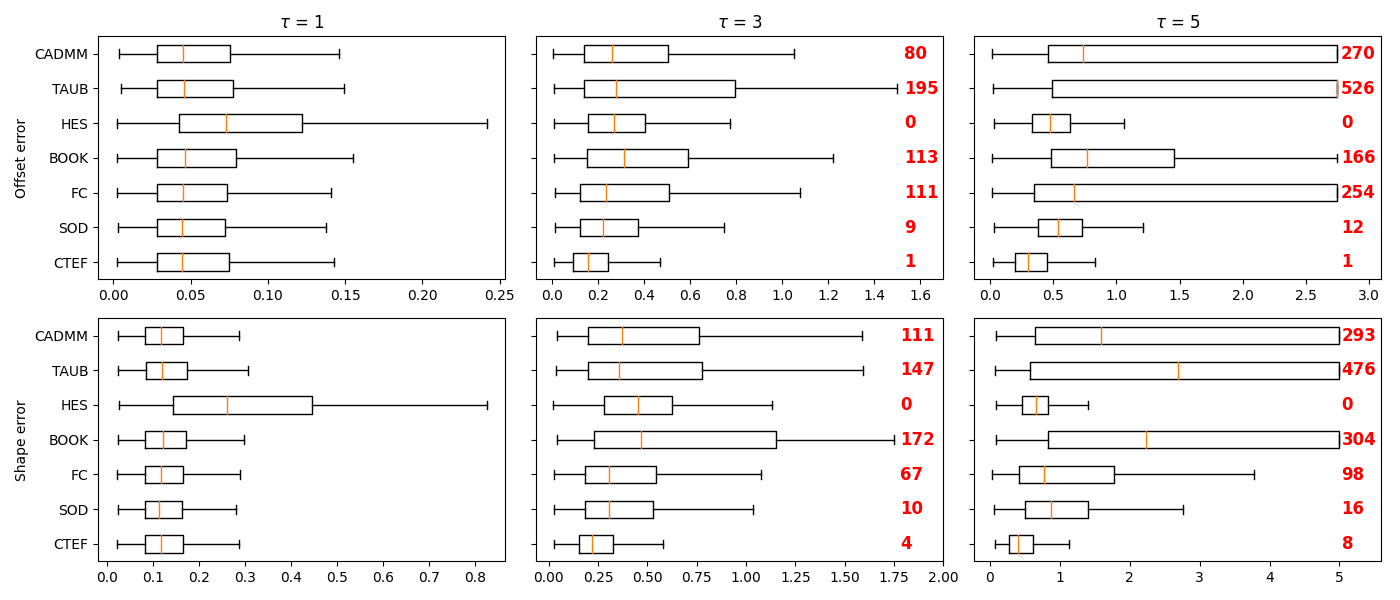}
\caption{\textit{Errors for different $\tau$ with $p=3$, $r_{ax}=2.5$, noise $=1\%$, and $n=18$. Only CTEF is stable for all values of $\tau$.}}
\label{fig:tau3}
\end{figure*}

\begin{figure*}
\centering
\includegraphics[width=\textwidth, height=.3\textheight]{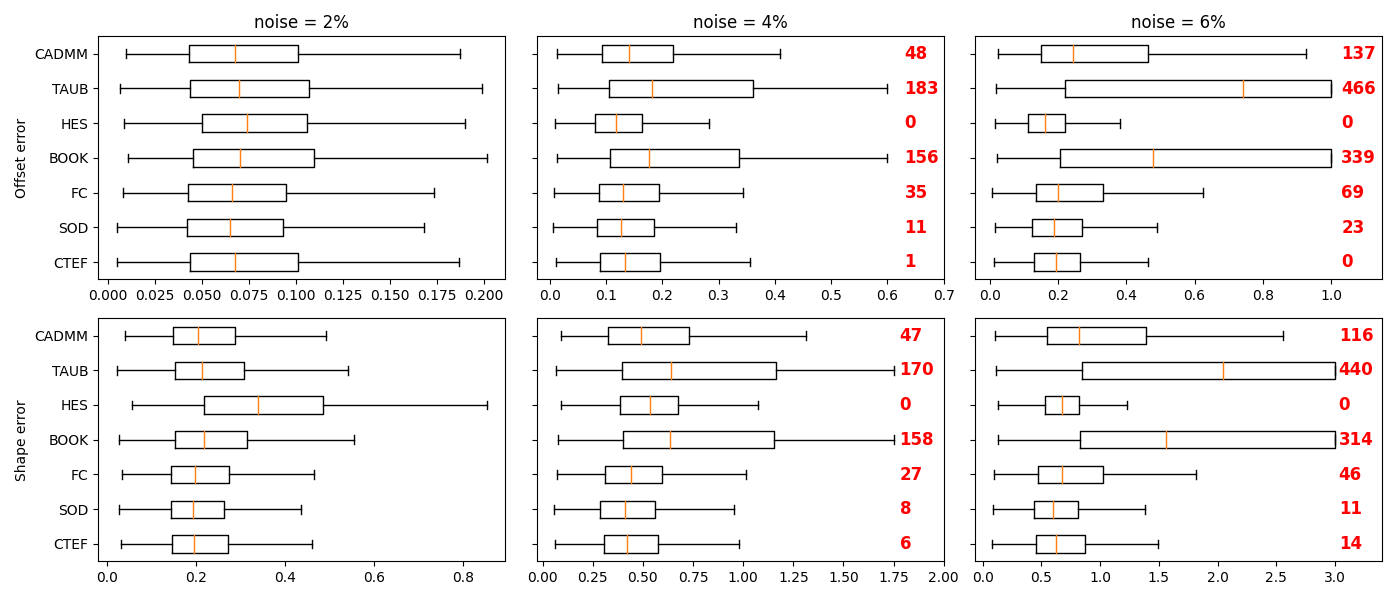}
\caption{\textit{Errors for different noise values with $p=3$, $\tau=0$, $r_{ax}=2.5$, and $n=18$. Only CTEF, SOD, and HES are stable as noise increases.}}
\label{fig:noise3}
\end{figure*}

\begin{figure*}
\centering
\includegraphics[width=\textwidth, height=.3\textheight]{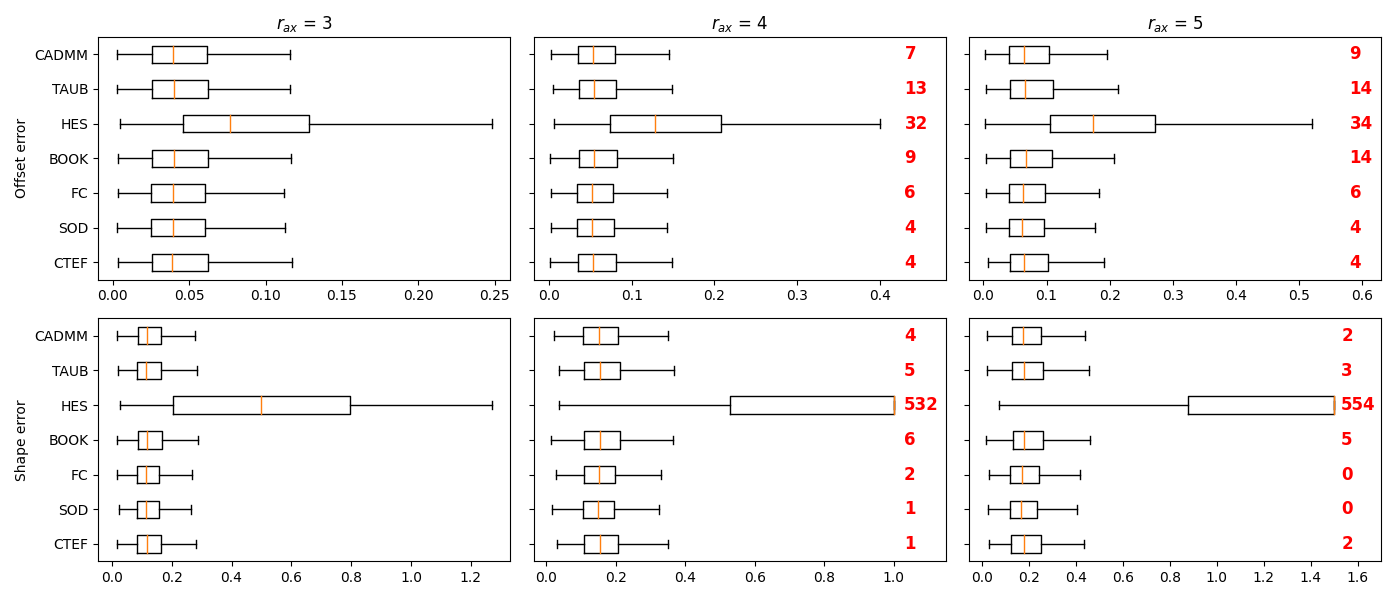}
\caption{\textit{Errors for different axis ratios with $p=3$, $\tau=0$, noise $=1\%$, and $n=18$. All methods except HES perform similarly.}}
\label{fig:ratio3}
\end{figure*}

\begin{figure*}
\centering
\includegraphics[width=\textwidth, height=.3\textheight]{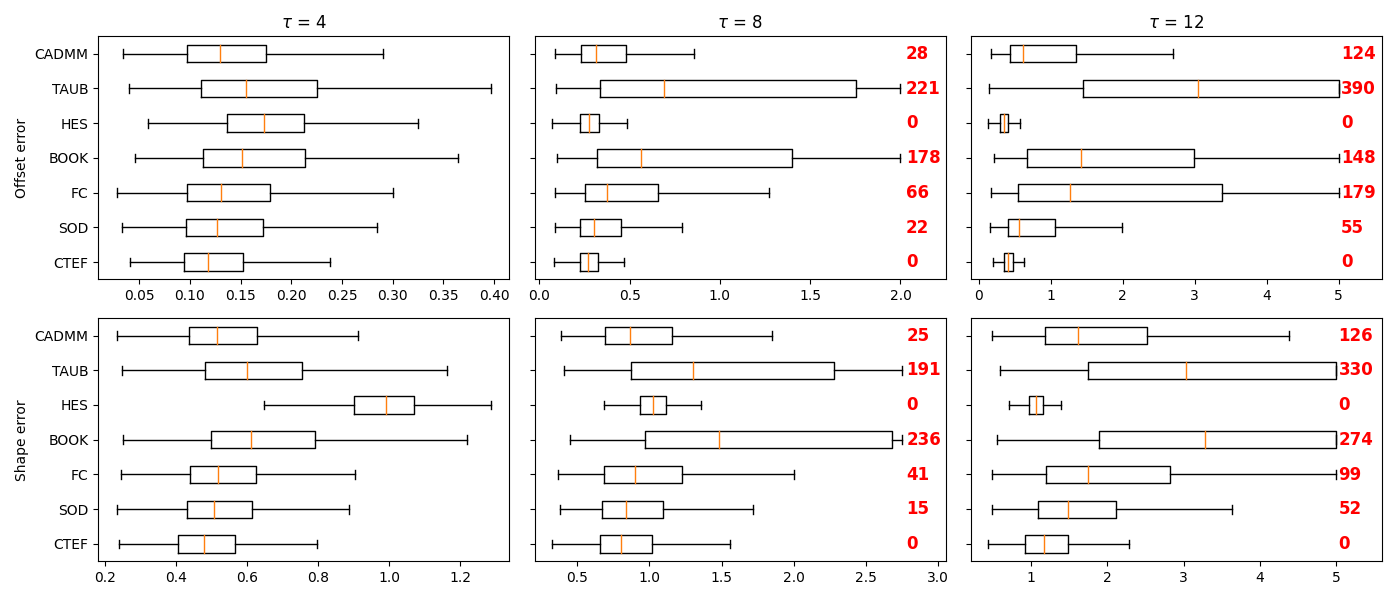}
\caption{\textit{Errors for different $\tau$ with $p=10$, $r_{ax}=2.5$, noise $=1\%$, and $n=100$. The $\tau$ values in this experiment are larger than those in the $p=3$ case (\Cref{fig:tau3}) to account for dependence between $\tau$ and dimension \cite{dhillon2003}.}}
\label{fig:tau10}
\end{figure*}

\begin{figure*}
\centering
\includegraphics[width=\textwidth, height=.27\textheight]{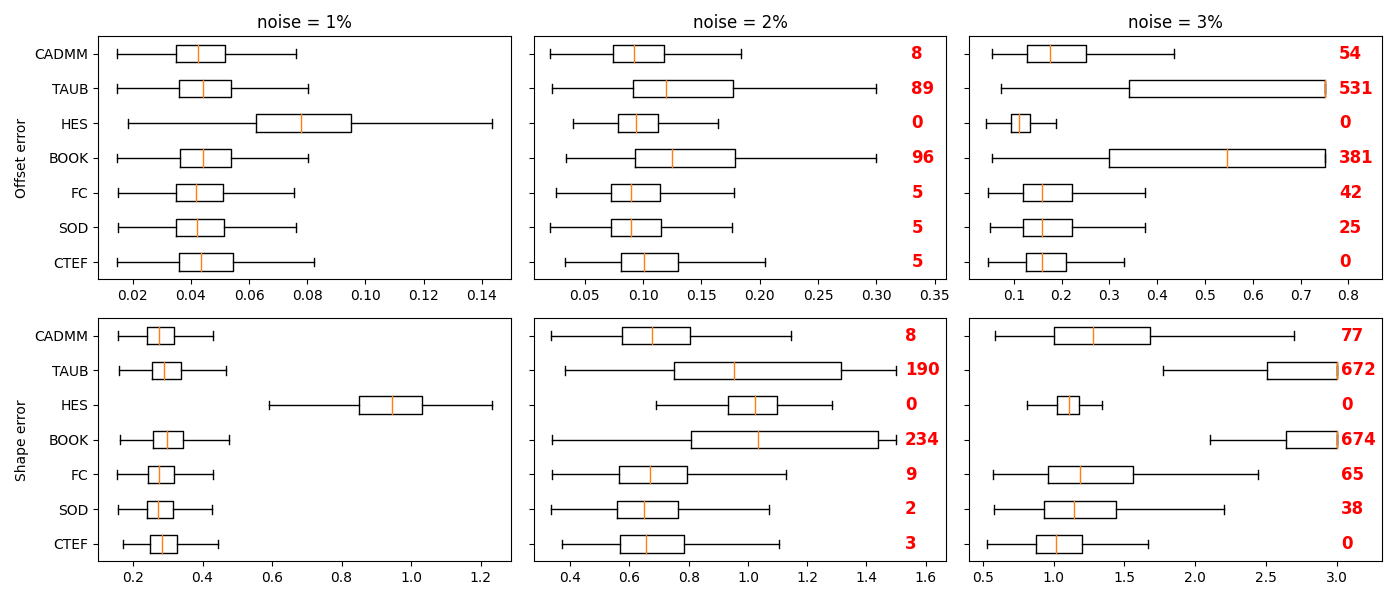}
\caption{\textit{Errors for different noise values with $p=10$, $\tau=0$, $r_{ax}=2.5$, and $n=130$.}}
\label{fig:noise10}
\end{figure*}


\subsection{Results.}\label{sec:results}


We consider the concentration parameter $\tau$, noise parameter $\sigma$, and axis ratio $r_{ax}$ which, recall, is the ratio between the longest and shortest axis lengths of an ellipsoid. In each experiment we vary the parameter of interest keeping others fixed. One experiment consists of $1{,}000$ trials. One trial consists of fitting an ellipsoid to data generated from the Ellipsoid-Gaussian model with the specified parameters using each fitting method. Box plots in \Cref{fig:tau3,fig:noise3,fig:ratio3,fig:tau10,fig:noise10} show distributions of the offset and shape errors for each fitting method in each experiment. Medians are indicated by the orange lines and whiskers extend to $1.5$ times the interquartile range. Red numbers in some plots indicate the number of times the error for a certain fitting method exceeds a specified threshold, meaning the method either had relatively large error or failed to return an ellipsoid at all. For example, in the upper right panel in \Cref{fig:tau3} the offset error for CADMM exceeds the cutoff value in $270$ of the $1{,}000$ trials.

\Cref{fig:tau3,fig:noise3,,fig:ratio3} correspond to $p=3$ and \Cref{fig:tau10,fig:noise10} to $p=10$. Except for HES, all methods perform similarly when data are uniformly distributed ($\tau=0$) with low noise. The most pronounced differences occur when data are not uniformly distributed (\Cref{fig:tau3,fig:tau10}). Only CTEF is stable throughout, a finding further validated by the Rosenbrock and circadian rhythm examples in \Cref{sec:rosenbrock} (\Cref{fig:rosenbrock_loss_compare}). 

The relative robustness of CTEF is due in part to the feasible set which uses the data to confine to ``reasonable" ellipsoid parameters and, in particular, prevents solutions from escaping to infinity. In all experiments $w_{ax}=10$ when $p=3$, $w_{ax}=1$ when $p=10$, and $w$ is automatically computed from $[c^-,c^+]$ via a linear regression model trained on data simulated from \eqref{eq:ellipsoid_gaussian}. Thus even with minimal user-involved parameter tuning CTEF outperforms other methods in the common setting of noisy, nonuniform data, making it the superior method in many practical settings.


\subsection{Parameter sensitivity}\label{sec:sensitivity}


\Cref{fig:sensitivity,fig:sensitivity_waxis} show sensitivity of CTEF to $w$ and $w_{ax}$ for different values of the variables $n$, $\tau$, $r_{ax}$, $\sigma$, and $p$. Specifically, for each variable we run $2{,}000$ trials with all other variables fixed. One trial consists of drawing the variable of interest uniformly at random from a specified range, simulating data from the Ellipsoid-Gaussian model with this and the fixed variables, fitting the data with CTEF for each parameter in a specified range, and computing the offset and shape errors for each fit. For example, in the leftmost plots of \Cref{fig:sensitivity}, each trial consists of drawing $n$ uniformly at random from $\{15,\dots,75\}$ with $\tau$, $r_{ax}$, $\sigma$, and $p$ fixed at $0$, $2$, $0.01$, and $3$, respectively. For each of $15$ evenly-spaced weights $w$ between $0.1$ and $3$ we apply CTEF with that weight and compute the resulting errors. The two plots show the average offset and shape errors as well as $\pm 1$ standard deviation over the $2{,}000$ trials. Our results indicate $w_{ax}$ has virtually no effect across all settings, and $w$ is significant predominantly when $\tau>2$ and, to a lesser extent, in the presence of considerable noise.

\begin{figure*}
\centering
\includegraphics[width=.75\textwidth, height=.2\textheight]{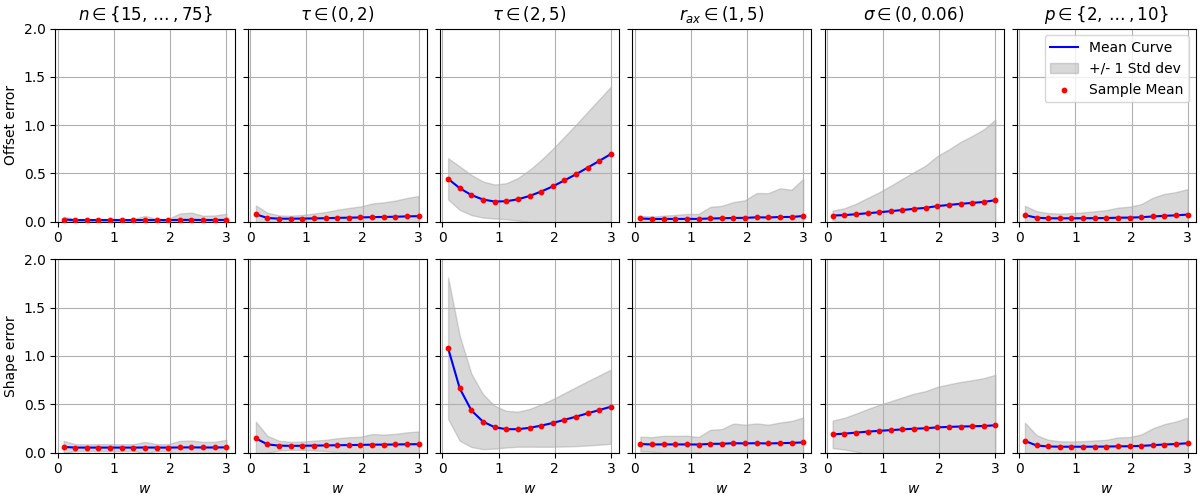}
\caption{\textit{Sensitivity of CTEF to $w$. The weight $w$ ranges over $15$ evenly spaced points between $0.1$ and $3$ and $w_{ax}=10$. Variables $n$, $\tau$, $r_{ax}$, $\sigma$, and $p$ are fixed at $30$, $0$, $2$, $0.01$, an $3$, respectively, when not varied. The only exception is $n=4p+p(p-1)$ when $p$ varies (rightmost plots).}}
\label{fig:sensitivity}
\end{figure*}

\begin{figure*}
\centering
\includegraphics[width=.75\textwidth, height=.2\textheight]{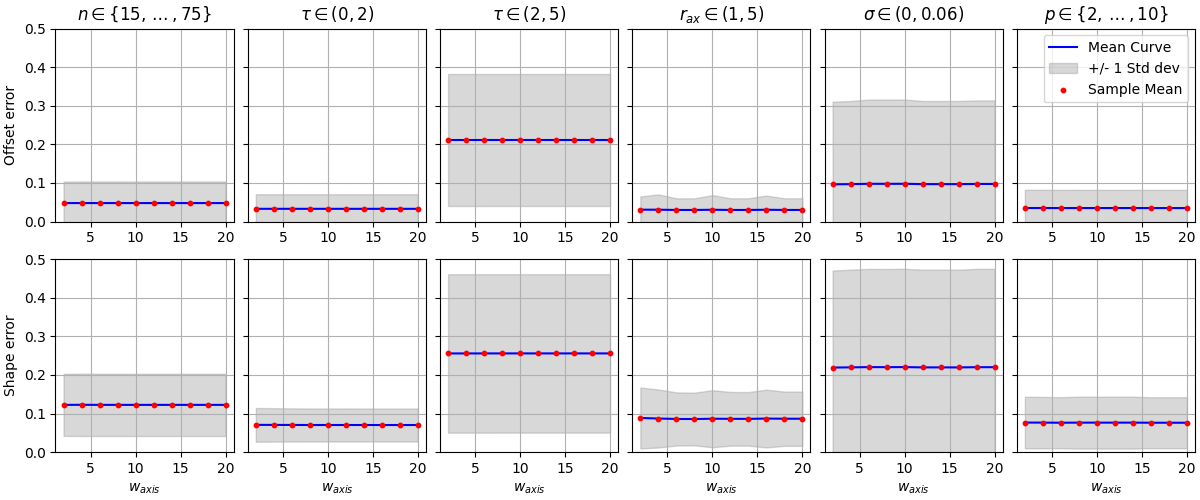}
\caption{\textit{Sensitivity of CTEF to $w_{ax}$. Here $w=1$ and $w_{ax}$ ranges over $\{2,4,6,\dots,20\}$. All other conditions are the same as in \Cref{fig:sensitivity}.}}
\label{fig:sensitivity_waxis}
\end{figure*}


\subsection{Runtime}\label{sec:runtime}


\Cref{tab:dimension,tab:n} report average runtime in milliseconds over $1{,}000$ trials in different dimensions and for various sample sizes. When $p$ varies, sample size is set to $4p+p(p-1)$ and in \Cref{tab:n} $p=3$. In these experiments all algorithms are run using MATLAB version 23.3.0. FC is the fastest method -- including among those not shown -- and CTEF is the slowest. Nevertheless, CTEF runs in approximately $1/10$th of $1$ second when $p=3$ and $n=1{,}000$ as well as in $15$ dimensions, a practical runtime in many settings including all those in \Cref{sec:applications}.

\begin{table}
\resizebox{.5\textwidth}{!}{%
\begin{tabular}{lccccc}
\hline
 Method   &   p = 2 &   p = 3 &   p = 5 &   p = 10 &   p = 15 \\
\hline
 CTEF     &   17.18 &   22.29 &   15.27 &    37.36 &    82.45 \\
 SOD      &    0.18 &    0.23 &    0.3  &     2.28 &    10.26 \\
 FC       &    0.08 &    0.12 &    0.15 &     0.57 &     1.37 \\
 CADMM    &    0.12 &    0.16 &    0.2  &     1.22 &     3.31 \\
\\
\end{tabular}
}
\caption{Runtime (in milliseconds) vs dimension}
\label{tab:dimension}
\end{table}

\begin{table}
\resizebox{.5\textwidth}{!}{%
\begin{tabular}{lcccc}
\hline
 Method   &   n = 100 &   n = 250 &   n = 500 &   n = 1000 \\
\hline
 CTEF     &   27.74 &     46.66 &     73.6  &     120.71 \\
 SOD      &      0.24 &      0.28 &      0.37 &       0.34 \\
 FC       &      0.11 &      0.13 &      0.14 &       0.16 \\
 CADMM    &     0.22 &      0.38 &      0.64 &       1.1  \\
 \\
\end{tabular}
}
\caption{Runtime (in milliseconds) vs sample size}
\label{tab:n}
\end{table}


\section{Applications}\label{sec:applications}


So far we fit $p$-dimensional ellipsoids to $p$-dimensional data $X=\{x^{(i)}\}$. For $k<p$ we can instead fit a $k$-dimensional ellipsoid to $X$ by replacing the $p$-by-$p$ matrix $V$ with a $p$-by-$k$ matrix $V_k$ whose $k$ columns are distinct columns of $V$. This is equivalent to applying principal component analysis (PCA) to $X$ and fitting an ellipsoid to the projection of $X$ onto $k$ principal components. The corresponding loss $\mathcal{L}_k$ is identical to \eqref{eq:loss} except now $y^{(i)}=V_k^T(x^{(i)}-\bar{x})\in\mathbb{R}^k$ and the feasible set -- which was a $(3+p)p/2$ dimensional rectangle -- is $(3+k)k/2$ dimensional. Given a solution ($a^k_*, c^k_*, s^k_*)$ of \eqref{eq:objective} with $\mathcal{L}_k$ in place of $\mathcal{L}$, the $k$-dimensional ellipsoid of best fit for $X$ is the embedded (in $\mathbb{R}^p$) $k$-dimensional ellipsoid
\begin{align*}
	\mathcal{E}^k_X &= \{V_k\widetilde{R}(s^k_*)^TA(a^k_*)^{-1}\eta + \widetilde{c}^k_*:\eta\in S^{k-1}\},
\end{align*}
where $\widetilde{R}(s^k_*)=R(s^k_*)V_k^T$ and $\widetilde{c}^k_*=V_kc^k_*+\bar{x}$. As in PCA, $V_k$ often consists of the first $k$ columns of $V$ which, recall, correspond to the $k$ largest eigenvalues of the covariance matrix of the centered data. However, the Rosenbrock example (\Cref{sec:rosenbrock}) shows the first $k$ principal components are not always optimal for ellipsoid fitting. In \Cref{sec:cell} we use this PCA-based method to visualize human cell cycle data, illustrating its potential usefulness when data have nonlinear features. 

Ellipsoid fitting can also be used to analyze, visualize, and reduce dimensionality of periodic data. Unlike the PCA-based method above, this approach does not leverage existing dimension reduction algorithms. To illustrate the idea, consider noisy observations from different periodic functions at discrete time steps. Our goal is to determine how synchronized, i.e. in phase, these functions are. For concreteness, suppose we have three sets of data $X_\omega=\{x^{(i)}_\omega = \cos(t_i + \omega)+\epsilon : t_i\in T\}$ with $\omega\in\{0,\pi/4,\pi\}$, $T=\{0,\dots,4\pi\}$, and $\epsilon\sim\mathcal{N}(0,0.1)$ (\Cref{fig:periodic}). To configure the data for ellipsoid fitting, set $\max_\omega=\max\{x^{(i)}_\omega\}$, $\min_\omega=\min\{x^{(i)}_\omega\}$, and define
\begin{align}\label{eq:scale}
	y^{(i)}_\omega &= \frac{x^{(i)}_\omega - \min_\omega}{\max_\omega} + 1 
\end{align}
so $\max\{y^{(i)}_\omega\}=2$ and $\min\{y^{(i)}_\omega\}=1$. Then map $(t_i,y^{(i)}_\omega)\mapsto (y^{(i)}_\omega\cos(t_i/2), y^{(i)}_\omega\sin(t_i/2))$, which is a polar transformation but with observations taken over twice the period of cosine, namely $4\pi$, hence the factor $1/2$. This is because we want the maximum of the data to occur twice in one cycle, with the two copies of the maximum ideally representing two ends of the major axis of an ellipse. Similarly, the two copies of the minimum ideally represent two ends of the minor axis. The first three columns in \Cref{fig:periodic} show the scaled (top row) and polar-transformed (bottom row) data for each value of $\omega$. The fourth and fifth columns show $X_0\cup X_{\pi/4}$ and $X_0\cup X_\pi$, respectively. The key observation is that synchronization is captured by the axis ratio of the ellipsoid of best fit for the combined data sets, with larger axis ratios corresponding to greater synchronization. In \Cref{fig:periodic} the axis ratios of the best fit ellipsoids for $X_0\cup X_{\pi/4}$ and $X_0\cup X_\pi$ are $1.751$ and $1.042$.

\begin{remark}
In the above example we assumed data were collected over twice the period of cosine. If one were to instead have observations from less than two periods, e.g. over $[0,2\pi]$, then the scaling \eqref{eq:scale} maps data to a proper subset, e.g. half, of an ellipsoid. Such situations are quite common, especially in biological applications where longitudinal data are often difficult and expensive to collect. This highlights the benefit of CTEF's superior ability to fit nonuniform data relative to other ellipsoid fitting algorithms; see in particular \Cref{fig:circadian_half}.
\end{remark}

\begin{figure*}
\centering
\includegraphics[width=.9\textwidth, height=.25\textheight]{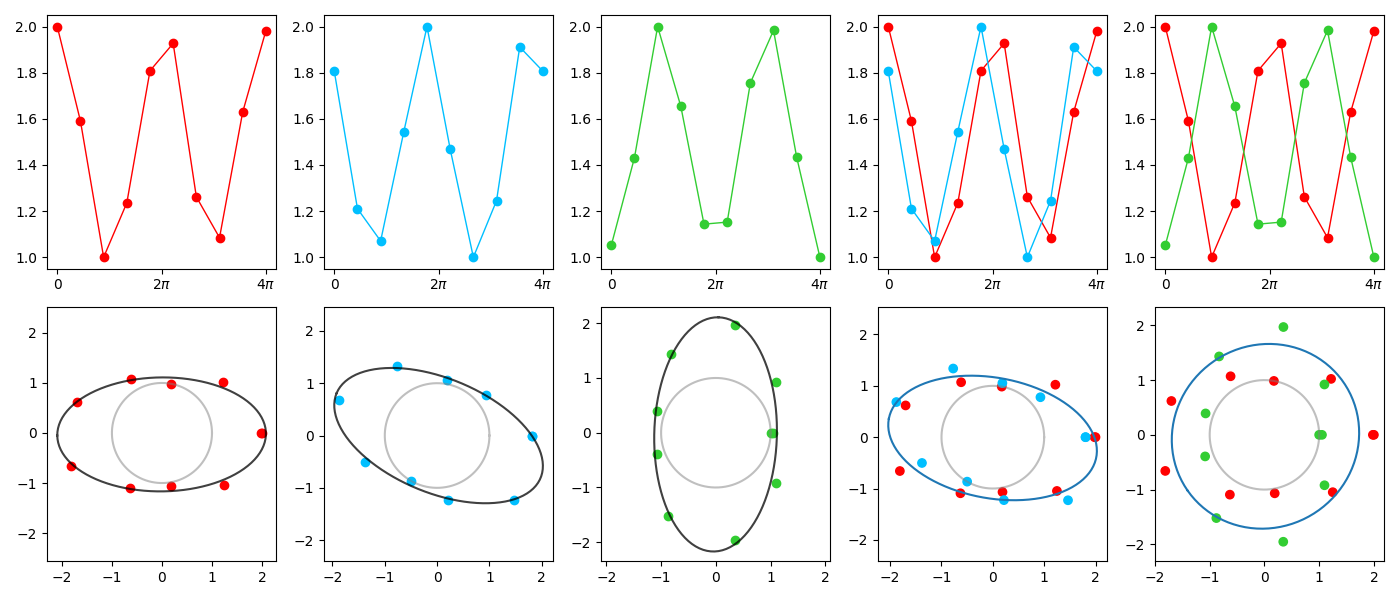}
\caption{\textit{(Left to right) Periodic data $X_0$, $X_{\pi/4}$, $X_\pi$, $X_0\cup X_{\pi/4}$, and $X_0\cup X_\pi$. The first row is the scaled data; the second is its polar transformation with the unit circle in gray for reference. Ellipsoids are fit with CTEF. Note the different axis ratios in the fourth and fifth columns.}}
\label{fig:periodic}
\end{figure*}


\subsection{Rosenbrock.}\label{sec:rosenbrock}


The Rosenbrock function \cite{rosenbrock1960} is often used to test optimization algorithms due to its simple yet numerically challenging geometry. In this example we fit an ellipsoid to data generated from a $3$-dimensional hybrid Rosenbrock model \cite{pagani2022} with parameters from \cite[Section 5.2]{song2022}. Specifically, $2{,}000$ data points $X=\{x^{(i)}\}$ are drawn from the probability density on $\mathbb{R}^3$ proportional to
\begin{align*}
	\exp\left(-x_1^2-30(x_2-x_1^2)^2-(x_3-x_2^2)^2\right).
\end{align*}
Exploratory analyses, namely visualization and PCA, indicate data concentrate near a $2$-dimensional ellipse. Our goal therefore is to fit $X$ with a $2$-dimensional ellipse using one of $V^{(1,2)}_2=(v_1,v_2)$, $V^{(1,3)}_2=(v_1,v_3)$, or $V^{(2,3)}_2=(v_2,v_3)$ in place of $V_k$ as described above, where $v_i$ is the $i$th principal component of the centered data. To choose which pair to use, we ran CTEF in each case and picked the $(i,j)$ that minimizes the loss \eqref{eq:loss} with $V^{(i,j)}_2$ in place of $V$. The left panel in \Cref{fig:rosenbrock_loss_compare} shows $V^{(1,3)}_2$ is best in this regard, a finding that is strongly supported by plotting the corresponding ellipsoids of best fit. In particular, using the first two principal components is suboptimal. This reflects the fact that, while variance in $X$ is maximized among all $2$-dimensional subspaces when data are projected onto the subspace spanned by $v_1$ and $v_2$, the most pronounced curvature in the PCA transformed data lies in the $(v_1,v_3)$ plane. 

\begin{figure*}
\centering
\includegraphics[width=.8\textwidth, height=.25\textheight]{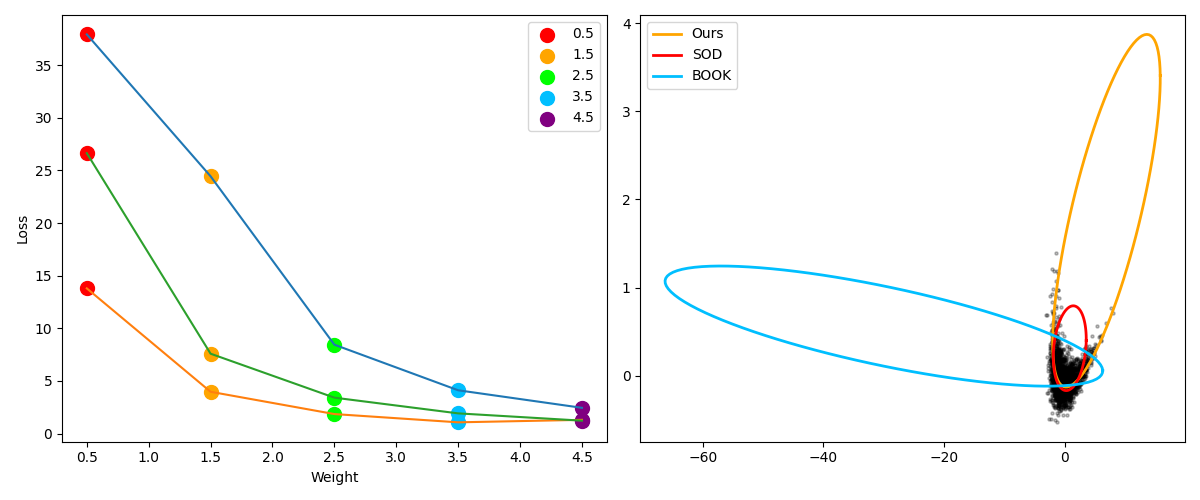}
\caption{\textit{(Left) Value of loss \eqref{eq:loss} corresponding to $V^{(1,2)}_2$ (blue line), $V^{(1,3)}_2$ (orange line), and $V^{(2,3)}_2$ (green line), for feasible set weights $w\in\{0.5,1.5,2.5,3.5,4.5\}$. $V^{(1,3)}_2$ yields the smallest loss for each weight. The ``scree" nature of this plot indicates the ellipses corresponding to $1.5$ or $2.5$ are likely best. (Right) Black dots are Rosenbrock data $X$ projected onto the subspace spanned by the first and third principal components, $v_1$ and $v_3$, of the data. We then fit ellipses using the methods discussed in this paper. FC is omitted because it returned an ellipse whose largest axis length was magnitudes larger than the ones shown. CADMM is omitted because it failed to return any solution at all. This agrees with our finding in \Cref{sec:experiments} that other methods perform poorly when data are not uniformly distributed over an ellipse.}}
\label{fig:rosenbrock_loss_compare}
\end{figure*}



\subsection{Cell cycle.}\label{sec:cell}


The \textit{cell cycle} is the process of cell division. It progresses through $4$ phases: $G_1$ (gap 1), $S$ (DNA synthesis), $G_2$ (gap 2), and $M$ (mitosis), then back to $G_1$. In \cite{stallaert2022} the authors identify $p=40$ core cell cycle features based on studies of $n=8{,}850$ individual cells. One of their principal goals is visualization. \Cref{fig:cell} shows two views of their data obtained as follows. First, let $X=\{x^{(i)}\}$ be the projection of the data onto its first $k=3$ principal components. Second, apply CTEF (or any ellipsoid fitting algorithm) to get the parameters $A$, $R$, and $c$ of the ellipsoid of best fit for $X$. The plotted points $\{y^{(i)}\}$ are then
\begin{align*}
	y^{(i)} &= R^TA^{-1}\left(\frac{AR(x^{(i)}-c)}{\lVert AR(x^{(i)}-c)\rVert}\right) + c.
\end{align*}
The term in parentheses is the closest unit vector to $AR(x^{(i)}-c)$. Multiplying by $R^TA^{-1}$ and translating by $c$ then maps it to the ellipsoid of best fit (\Cref{eq:ellipsoid}).

In addition to accurately grouping the $4$ cell cycle phases and preserving the cyclic ordering $G_1\rightarrow S\rightarrow G_2\rightarrow M\rightarrow G_1$, the fit in \Cref{fig:cell} also captures relative variability of each phase, with $G_1$ the longest lasting, $S$ and $G_2$ having similar time frames, and $M$ the most brief \cite{israels2000}. \Cref{fig:cell_compare} shows $2$-dimensional representations of the data using SOD and BOOK in place of CTEF, as well as other dimension reduction methods, namely PCA, tSNE \cite{vanderMaaten2008}, UMAP \cite{mcinnes2018}, Isomap \cite{balasubramanian2002}, and locally linear embedding (LLE) \cite{roweis2000}. In \Cref{sec:discussion} we discuss several of these methods and compare them to ours.

\begin{figure*}
\centering
\includegraphics[trim={0 7em 0 0}, clip, width=.8\textwidth, height=.4\textheight]{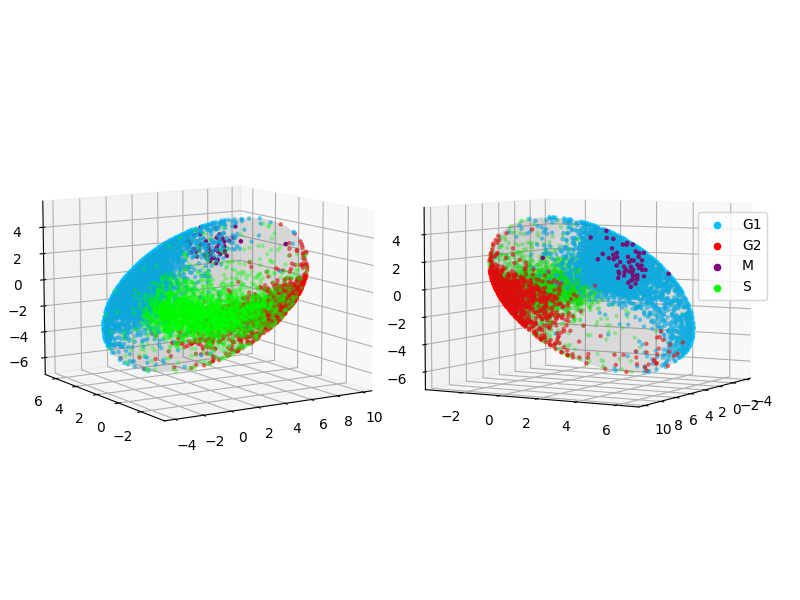}
\caption{\textit{Two views of $40$-dimensional cell cycle data fitted to an ellipsoid ($n=8{,}850$).}}
\label{fig:cell}
\end{figure*}

\begin{figure*}
\centering
\includegraphics[width=1\textwidth, height=.35\textheight]{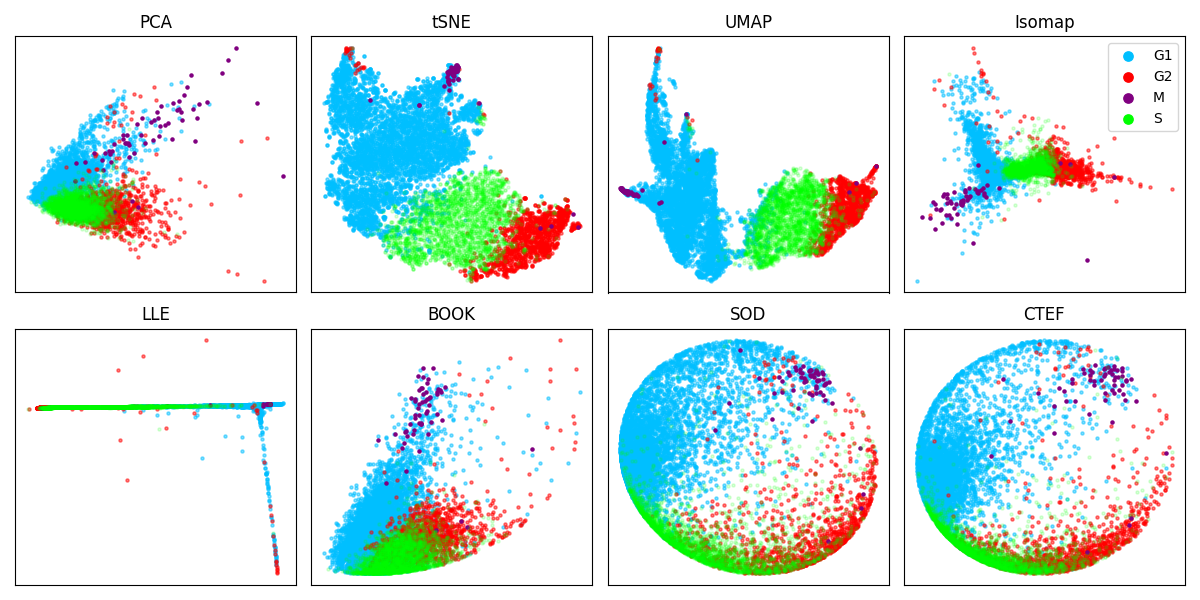}
\caption{\textit{Different dimension reduction techniques applied to $40$-dimensional cell cycle data. Plots for the ellipsoid fitting methods (BOOK, SOD, and CTEF) are obtained by projecting the $3$-dimensional fits described in \Cref{sec:cell} onto the first two standard coordinate axes of $\mathbb{R}^3$. CTEF, SOD, and the omitted FC, HES, and CADMM all perform similarly on this data. BOOK returns a long, narrow ellipsoid ill-suited to visualization and TAUB (not shown) fails to return an elliptic solution.}}
\label{fig:cell_compare}
\end{figure*}


\subsection{Circadian rhythm.}\label{sec:circadian}


Circadian rhythms are internal biological processes that cycle in approximately $24$ hour periods. They are found in all organisms and play a significant role in health and medicine, including in metabolic disorders, cardiovascular disease, and cancer \cite{rijo2019}. In \cite{zhang2014} the authors use RNA-sequencing to measure expression levels of $37{,}310$ genes in $12$ mouse organs\footnote{Adrenal gland, aorta, brainstem, brown fat, cerebellum, heart, hypothalamus, kidney, liver, lung, skeletal muscle, white fat} once every $6$ hours for $42$ hours. That is, for each gene and each organ there are $8$ observations of that gene's expression level, one at $0$, $6$, $12$, $18$, $24$, $30$, $36$, and $42$ hours. Genes whose expression levels oscillate with a period of $24$ hours are called \textit{circadian}. While each individual organ was found to have numerous circadian genes, e.g. $3{,}186$ in the liver, only $10$ circadian genes\footnote{Arntl, Dbp, Nr1d1, Nr1d2, Per1, Per2, Per3, Tsc22d3, Tspan4, Usp2} were found to synchronize across all $12$ organs. \Cref{fig:two_genes} shows data for one of these $10$ genes, Per3, and a randomly selected gene that is not one of the $10$. In \Cref{fig:circadian_full,fig:circadian_half} we apply the ellipsoid method for periodic data discussed above with various fitting algorithms (HES is omitted because it is identical to SOD in $2$ dimensions \cite{kesaniemi2018}). \Cref{fig:circadian_full} shows most methods perform similarly and that ellipsoids fit to the $10$ distinguished genes (red dots) have larger axis ratios than nearly all of the $1{,}000$ randomly selected genes. When only half the time-points are retained in \Cref{fig:circadian_half} (one $24$ hour cycle) only CTEF and, to a lesser extent, SOD perform well, which agrees with our experiments in \Cref{sec:results}.

\begin{figure*}
\centering
\includegraphics[width=.875\textwidth, height=.35\textheight]{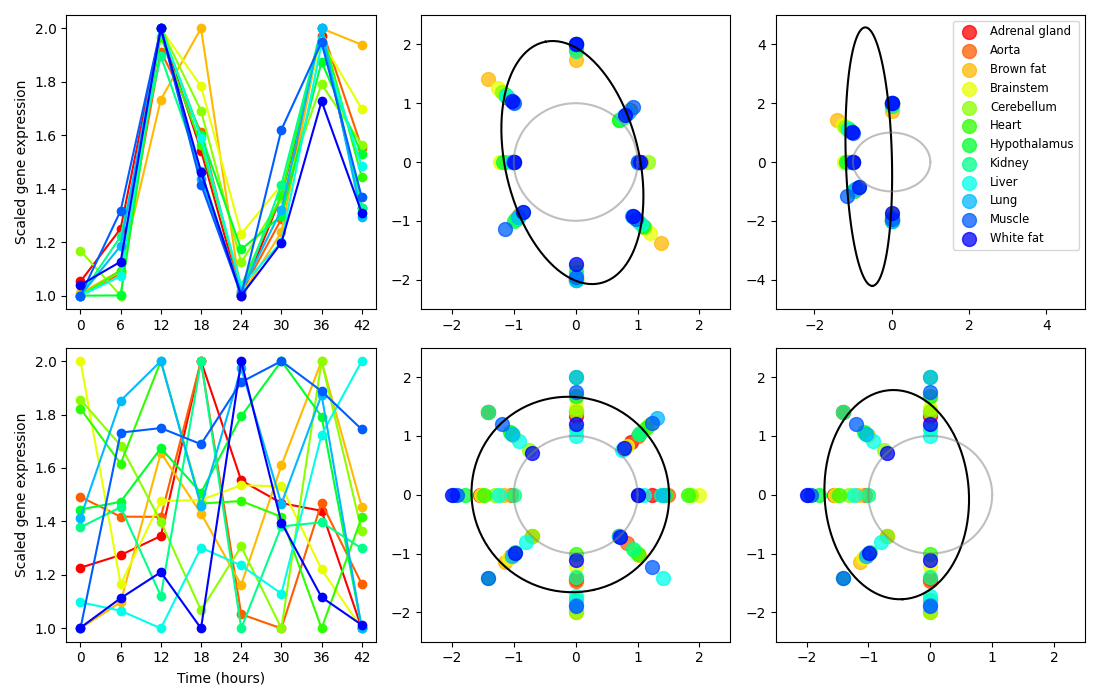}
\caption{\textit{(Left) Scaled gene expression levels for two genes over a $42$ hour period. (Center) Polar representations of the full data. (Right) Polar representations of the data between $12$ and $36$ hours. Ellipsoids are fit with CTEF. Colors represent different organs. The first row is one of the genes (Per3) that synchronizes across all organs; the second row is a gene selected at random. The unit circle is in gray for reference.}}
\label{fig:two_genes}
\end{figure*}

\begin{figure*}
\centering
\includegraphics[width=\textwidth, height=.25\textheight]{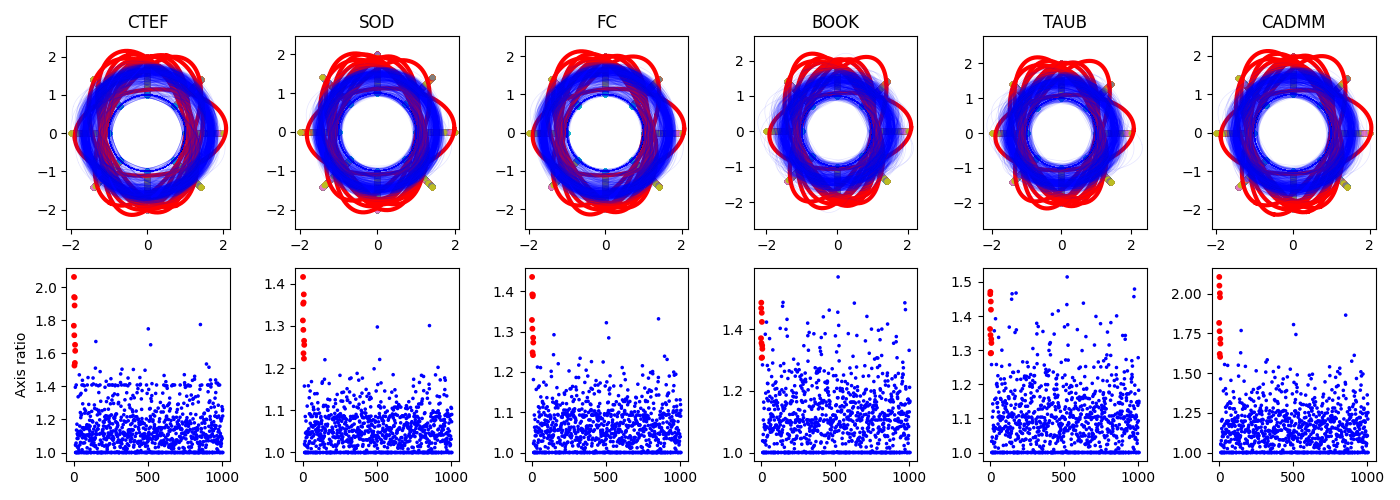}
\caption{\textit{Results from the analysis in \Cref{fig:two_genes} but for $1{,}010$ genes. (Top) Ellipsoids fit to each gene using various fitting methods. (Bottom) Axis ratios for each of the $1{,}010$ fits. In all figures red corresponds to the $10$ synchronized genes, and blue to the $1{,}000$ randomly selected ones. The gene Nfil3 with the largest axis ratio among the randomly selected ones for CTEF, SOD, FC, and CADMM (index $857$ in the plots) is also a known circadian gene \cite{kubo2020}. Indeed, plotting the data for this gene as in \Cref{fig:two_genes} indicates Nfil3 also synchronizes across all $12$ organs.}}
\label{fig:circadian_full}
\end{figure*}

\begin{figure*}
\centering
\includegraphics[width=\textwidth, height=.25\textheight]{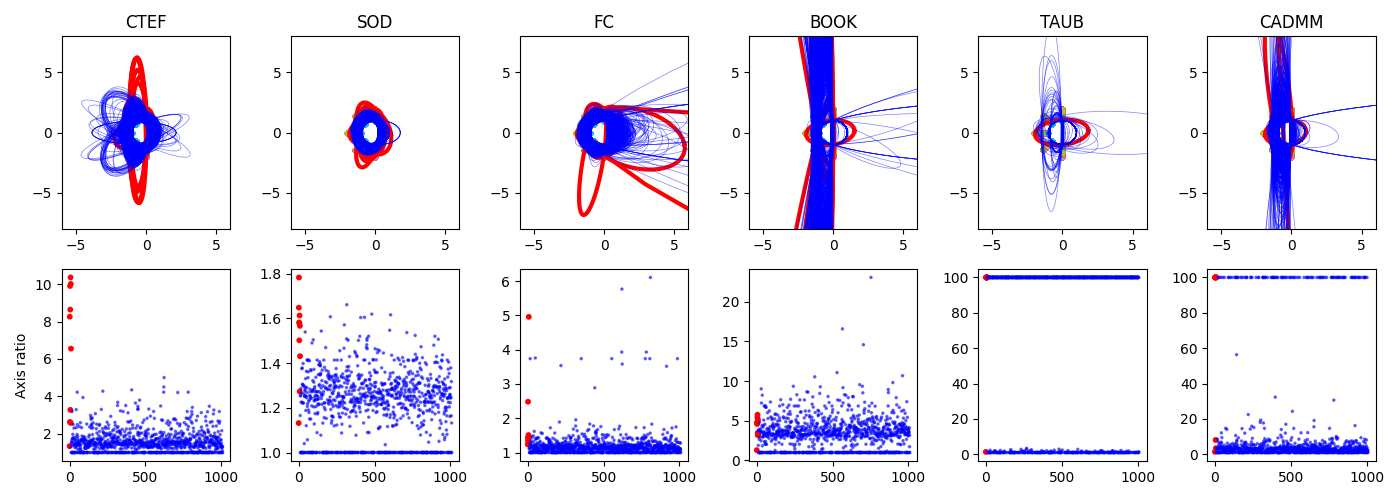}
\caption{\textit{The same plots as in \Cref{fig:circadian_full}, but with only observations at hours $12$, $18$, $24$, $30$, and $36$ included. Due to its ability to fit partial data, CTEF is still able to identify $6$ of the $10$ synchronized genes. The next best method in this regard, SOD, only definitively identifies $1$. The cutoff axis ratio value of $100$ shows TAUB and CADMM are frequently degenerate in this nonuniform data setting. Here $w_{ax}=1.5$.}}
\label{fig:circadian_half}
\end{figure*}


\subsection{Clustering}\label{sec:clustering}


\Cref{alg:clustering} gives an ellipsoid-based clustering algorithm similar to \cite[Algorithm 2]{paul2020}. The procedure is simple: Given $X=\{x^{(i)}\}_{i=1}^n\subseteq\mathbb{R}^p$ with cluster assignments $\kappa=(\kappa_1,\dots,\kappa_n)$ (so $x^{(i)}$ belongs to cluster $\kappa_i$), at each step do (1) fit an ellipsoid $\mathcal{E}_j=(A_j,R_j,c_j)$ to data $\{x_i:\kappa_i=j\}$ in the $j$th cluster, then (2) for each $i$ and $j$ assign $x^{(i)}$ to the cluster that minimizes the residual error $(\lVert A_jR_j(x^{(n)}-c_j)\rVert^2-1)^2$. Since the emphasis of this paper is on ellipsoid fitting, we leave detailed analysis of \Cref{alg:clustering} to future work. However, preliminary experiments shown in \Cref{fig:cluster2,,fig:cluster3} are promising: Ours is the only algorithm that gives good results in both cases relative to clustering algorithms implemented by the scikit-learn Python package \cite{sklearn2011}.

\begin{algorithm}
\caption{Ellipsoid clustering}\label{alg:clustering}
\begin{algorithmic}[1]
\Require{Data $X\in\mathbb{R}^{n\times p}$, number of clusters $n_c$, number of steps $n_s$}
\State Initialize clusters $\kappa=(\kappa_1,\dots,\kappa_n)\in\{1,\dots,n_c\}^n$
\State $i \gets 0$
\While{$i < n_s$}
	\For{$0\leq j < n_c$}
		\State $\mathcal{E}_j\gets \text{ellipsoid\_fit}(\{x_m:\kappa_m=j\})$
	\EndFor
	\For{$1\leq m\leq n$}
		\State $\kappa_m\gets \argmin_j \left(\lVert A_jR_j(x_m-c_j)\rVert^2-1\rVert\right)^2$
	\EndFor
	\State $i\gets i+1$
\EndWhile
\Ensure{Clusters $\kappa$, parameters $\{A_j,R_j,c_j\}_{j=1}^{n_c}$}
\end{algorithmic}
\end{algorithm}

%

\begin{figure*}
\centering
\includegraphics[width=\textwidth, height=.18\textheight]{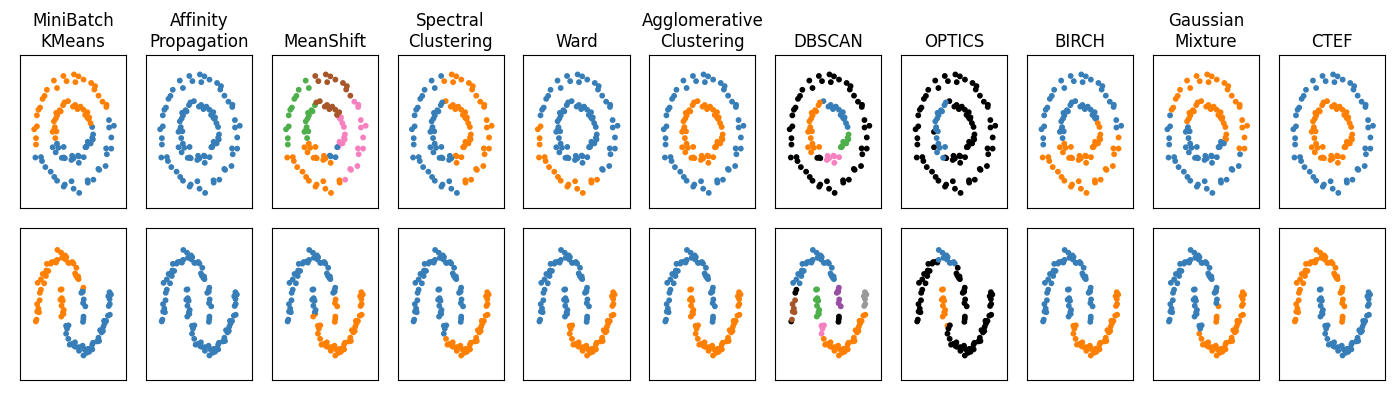}
\caption{\textit{Clustering $100$ samples from noisy circles and moons data. Code and results for all algorithms other than CTEF are exactly from \href{https://scikit-learn.org/stable/auto_examples/cluster/plot_cluster_comparison.html}{https://scikit-learn.org/stable/auto\_examples/cluster/plot\_cluster\_comparison.html} \cite{sklearn2011}. They state ``parameters of each of these dataset-algorithm pairs has been tuned to produce good clustering results." Despite this, only agglomerative clustering and CTEF accurately cluster both datasets.}}
\label{fig:cluster2}
\end{figure*}

\begin{figure*}
\centering
\includegraphics[width=\textwidth, height=.18\textheight]{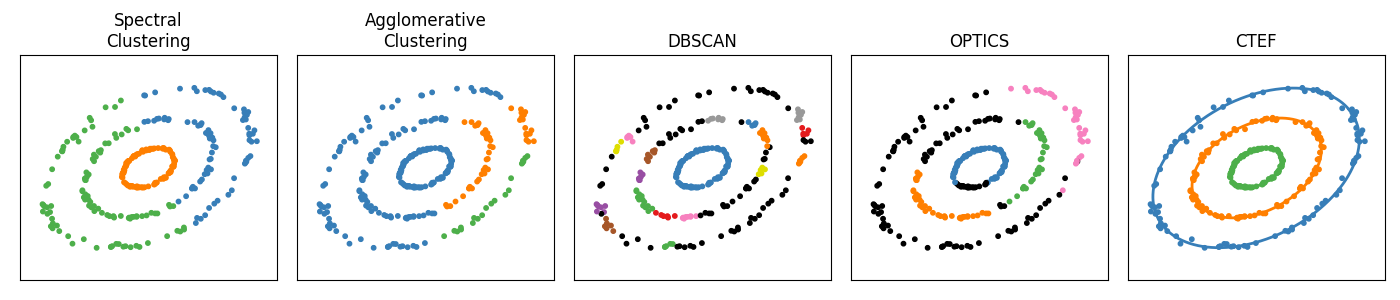}
\caption{\textit{Clustering $300$ samples from $3$ noisy ellipses. Methods in \Cref{fig:cluster2} not shown here performed poorly on these data and are omitted solely for ease of visualization. The $3$ ellipsoids returned by CTEF clustering are pictured in the CTEF panel. Thus, in addition to better clustering, CTEF returns a simple, closed-form parametrization of the ellipses corresponding to each cluster.}}
\label{fig:cluster3}
\end{figure*}


\section{Discussion}\label{sec:discussion}


We have introduced CTEF, a method that uses the Cayley transform to recast ellipsoid fitting as a bound-constrained minimization problem in Euclidean space. Unlike many existing methods, CTEF is ellipsoid specific, can fit arbitrary ellipsoids, and is invariant under rotations and translations of data. In addition, experiments in \Cref{sec:experiments} and the Rosenbrock and circadian rhythm examples in \Cref{sec:applications} indicate CTEF significantly outperforms other fitting methods when data concentrate near a proper subset of an ellipsoid.

In \Cref{sec:applications} we saw examples of how CTEF -- and ellipsoid fitting in general -- can be used for dimension reduction, visualization, and clustering. Aside from the number of components in PCA-based dimension reduction and the number of clusters in our clustering algorithm, the only tuning parameters in all cases were $w$ and $w_{ax}$ which have the clear geometric interpretations of rectangle size in which the ellipsoid center must lie and maximum axis length, respectively. Furthermore, CTEF is robust to $w$ and especially to $w_{ax}$ with $w$ reliably and automatically tuned with a simple linear regression model and $w_{ax}$ often requiring no tuning at all. This stands in stark contrast to many popular dimension reduction (DR) and clustering algorithms that require careful tuning of parameters. As stated recently in \cite{huang2022}, ``\textit{DR methods often become widely used without being carefully evaluated, and these methods may contain flaws that are unknown to their users. As further evidence of this, there are now many papers explaining how to use various popular DR methods effectively...These papers are only necessary because DR results are often misleading, and because DR cannot be trusted out-of-the box. These papers...highlight the urgent need to develop trustworthy DR methods.}" The authors go on to demonstrate tSNE and UMAP can be highly sensitive to parameter choices, can behave differently under different random seeds (making results difficult to reproduce), and can fail to capture global structure in data. Since CTEF is deterministic it is entirely reproducible; once the feasible set and initial condition are specified, it always returns the same result. Furthermore, as our examples show, ellipsoid fitting methods can capture global curvature and/or cyclic patterns in data. This is further evidenced by the relative success of ellipsoid clustering in \Cref{sec:clustering}. Like dimension reduction, many clustering algorithms are sensitive to potentially unintuitive parameter choices and can fail to capture global structure due to dependence on local parameter specifications such as number of nearest neighbors. For example, it is particularly clear in \Cref{fig:cluster3} that all methods besides ours only capture local curvature, failing to distinguish the three globally distinct ellipses that are abundantly clear to the naked eye. 


\subsection{Future work}\label{sec:future}


We see two primary directions for future work. In terms of the fitting algorithm itself, there is likely room for improvement in computational efficiency. In \cite{kesaniemi2018} the authors suggest combining direct methods such as SOD, FC, and BOOK with iterative methods such as CADMM and now CTEF. Since direct methods are faster, this could reduce the time complexity. It would also be useful to further investigate the feasible set, specifically the weight $w$, which can play an important role when data are nonuniform. While the linear regression model for choosing $w$ worked well in experiments, there are likely better models for automatic selection of $w$.

The other main line of future work is to extend the results of \Cref{sec:applications}. While we make no claim that ellipsoid fitting is a silver bullet -- surely there are data on which it would perform poorly -- our examples indicate it is a promising tool for capturing nonlinearities and global features in data that are missed by many popular dimension reduction and clustering algorithms. As an immediate next step it is interesting to test CTEF on various data that are known or suspected to be curved or cyclic in nature. In the context of dimension reduction, one can replace our initial PCA step with a different dimension reduction method before ellipsoid fitting. Similarly, we initiate our clustering algorithm with $k$-means. An obvious modification is to initiate with a different clustering algorithm to improve performance. Finally, it is good to establish further theory for ellipsoid fitting as a dimension reduction and clustering method, both to obtain rigorous guarantees and shed light on practical limitations and viable enhancements. \\

\textit{Code availability}. Code for this work is available at \href{https://github.com/omelikechi/ctef}{https://github.com/omelikechi/ctef}.


\section*{Acknowledgments}


The authors thank Andrea Agazzi, Didong Li, Katerina Papagiannouli, and Hanyu Song for helpful conversations and the authors of \cite{kesaniemi2018,lin2016} for providing their code. This project has received funding from the European Research Council (ERC) under the European Union’s Horizon 2020 research and innovation programme (grant agreement No 856506). It is also supported by National Institutes of Health grants R01-ES027498 and R01-ES028804 and Office of Naval Research grant 00014-21-1-2510-P00001. OM also thanks NSF-DMS-2038056 for partial support.

\bibliographystyle{abbrv}
\bibliography{/Users/omarmelikechi/iCloud/math/library/refs/refs.bib}


\end{document}